\pdfoutput=1

\documentclass[11pt,x11names]{article}

\usepackage{misc/acl2023}

\usepackage{times}
\usepackage{comment}
\usepackage{xfrac}
\usepackage{xcolor}
\usepackage{latexsym}
\usepackage{enumerate}
\usepackage{enumitem}
\setenumerate[1]{label=(\arabic*)}
\usepackage{listings}
\usepackage{algorithm, algorithmicx, algpseudocode}

\newcommand{\algcolorB}{PaleGreen3!80}
\newcommand{\algcolorC}{Khaki2!80}
\newcommand{\colorsquare}[1]{\resizebox{3mm}{!}{\colorbox{#1}{\textcolor{#1}{X}}}}

\usepackage[T1]{fontenc}
\usepackage[utf8]{inputenc}
\usepackage{inconsolata}
\usepackage{tikz}
\usetikzlibrary{arrows}

\usepackage{comment}
\usepackage{amsfonts}
\usepackage{amsmath}
\usepackage{amssymb,bm}
\usepackage{amsthm}
\usepackage{mathtools}
\usepackage{xspace}
\DeclarePairedDelimiter\ceil{\lceil}{\rceil}

\usepackage{enumitem}
\usepackage{algpseudocode}
\usepackage{microtype}
\usepackage{ulem}
\usepackage{booktabs}
\usepackage{multicol,multirow}
\usepackage{soul} %
\usepackage{thmtools}
\usepackage{thm-restate}
\usepackage{xspace}
\usepackage{scalerel}
\usepackage{array}

\allowdisplaybreaks

\DeclareMathOperator*{\expect}{\mathbb{E}}
\newcommand{\expectdelta}{\expect_{\hspace{-0mm}\bdelta \sim \pdeltastar\hspace{-2mm}}}

\newcommand{\bleu}{\textsc{Bleu}\xspace}
\newcommand{\chrf}{\textsc{ChrF}\xspace}
\newcommand{\comet}{\textsc{COMET}\xspace}
\newcommand{\bleurt}{\textsc{BLEURT}\xspace}
\newcommand{\performance}{\ensuremath{\textsc{Performance}_M}\xspace}

\newcommand{\strlen}{L}
\newcommand{\avgencodinglen}{\overline{L}_{\enc{\alphabettwo}}}
\newcommand{\avgencodinglenuni}{\overline{L}_{\encuni{\alphabettwo}}}
\newcommand{\avgencodinglenopt}{\overline{L}_{\encopt{\alphabettwo}}}
\newcommand{\avgencodinglenopts}[1]{\overline{L}_{\encopts{#1}{\alphabettwo}}}
\usepackage[capitalize,noabbrev]{cleveref}
\Crefname{section}{\S}{\S\S}
\Crefname{table}{Tab.}{}
\Crefname{figure}{Fig.}{}
\Crefname{algorithm}{Alg.}{}
\Crefname{equation}{Eq.}{Eq.}
\Crefname{definition}{Definition}{Definition}
\Crefname{appendix}{App.}{}
\Crefname{theorem}{Theorem}{}
\Crefname{myexample}{Example}{}
\DeclareCaptionType{myexamplef}[Example][List of Examples]
\Crefname{mynote}{Note}{}
\Crefname{prop}{Proposition}{}
\Crefname{cor}{Corollary}{}
\Crefname{observation}{Observation}{}
\Crefname{assumption}{Assumption}{}
\Crefname{hypothesis}{Hyp.}{Hypotheses}
\Crefformat{section}{\S#2#1#3}
\newcommand{\defeq}[0]{\mathrel{\stackrel{\textnormal{\tiny def}}{=}}}

\theoremstyle{plain}
\newtheorem{theorem}{Theorem}[section]

\newtheorem{myexample}[theorem]{Example}

\newtheorem{lemma}[theorem]{Lemma}
\newtheorem{hypothesis}[theorem]{Hypothesis}
\newtheorem{corollary}[theorem]{Corollary}

\newtheorem{definition}[theorem]{Definition}

\theoremstyle{remark}

\newcommand{\defn}[1]{\textbf{#1}}
\newcommand{\prooftext}[1]{\text{\color{black!50}#1}}

\usepackage{todonotes}
\makeatletter
\newcommand*\iftodonotes{\if@todonotes@disabled\expandafter\@secondoftwo\else\expandafter\@firstoftwo\fi} 
\makeatother

\newcommand{\spacesymb}{$\rule{5pt}{1.5pt}\hspace{-5pt}\textcolor{white}{\_}\xspace$}
\newcommand{\spacesymbB}{$\rule{5pt}{1.5pt}\hspace{-5pt}\textcolor{white}{\_}$}

\newcommand{\alphabetone}{\Sigma}
\newcommand{\alphabettwo}{\Delta}
\newcommand{\psigma}{p_{\scaleto{\alphabetone}{4pt}}}
\newcommand{\pdelta}{p_{\scaleto{\alphabettwo}{4pt}}}
\newcommand{\psigmastar}{p_{\scaleto{\alphabetone^*}{4pt}}}

\newcommand{\pdeltastar}{p_{\scaleto{\alphabettwo^*}{4pt}}}

\newcommand{\psigmastarempirical}{\ensuremath{\widehat{\psigmastar}}}
\newcommand{\countc}{\mathrm{count}}
\newcommand{\Lexpect}{\mathbb{E}[L]}
\newcommand{\ecount}{\mathrm{E\text{-}count}(\delta)}
\newcommand{\stext}{\boldsymbol{\sigma}}
\newcommand{\ssymbol}{\delta}
\newcommand{\ssymbols}{\boldsymbol{\delta}}
\newcommand{\character}{\sigma}

\newcommand{\bdelta}{\boldsymbol{\delta}}
\newcommand{\bsigma}{\boldsymbol{\sigma}}

\newcommand{\rvDelta}{\mathrm{W}_{\scaleto{\alphabettwo}{4pt}}}

\newcommand{\ent}{\mathrm{H}}
\newcommand{\encn}{{\normalfont \small \textsf{enc}}}
\newcommand{\encoptn}{\encn^\star}
\newcommand{\enc}[1]{{\normalfont \small \textsf{enc}_{\scaleto{#1}{4pt}}}}
\newcommand{\encopt}[1]{{\normalfont \small \textsf{enc}^\star_{\scaleto{#1}{4pt}}}}
\newcommand{\encopts}[2]{{\normalfont \small \textsf{enc}^{#1}_{\scaleto{#2}{4pt}}}}
\newcommand{\encuni}[1]{{\normalfont \small \textsf{enc}^{\scaleto{U}{4pt}}_{\scaleto{#1}{4pt}}}}

\newcommand{\eff}{{\small \textsf{eff}}}
\newcommand{\Lwenconedelta}{\ensuremath{\mathcal{L}_\enc{\alphabettwo}(\pdelta)}}
\newcommand{\Lwenconedeltaopt}{\ensuremath{\mathcal{L}_{\encopt{\alphabettwo}}(\pdelta)}}
\newcommand{\Lwenc}[1]{\ensuremath{\mathcal{L}^{(#1)}_\enc{\alphabettwo}(\pdeltastar)}}
\newcommand{\Lwenctwo}[1]{\ensuremath{\mathcal{L}_{\encopts{#1}{\alphabettwo}}(\pdeltastar)}}
\newcommand{\Lwencone}{\ensuremath{\mathcal{L}_\enc{\alphabettwo}(\pdeltastar)}}
\newcommand{\Lwenconeopt}{\ensuremath{\mathcal{L}_{\encopt{\alphabettwo}}(\pdeltastar)}}
\newcommand{\Lwenconeoptstar}{\ensuremath{\mathcal{L}_{\encoptn}(\pdeltastar)}}
\newcommand{\Lwencdelta}[1]{\ensuremath{\mathcal{L}^{(#1)}_\enc{\alphabettwo}(\pdelta)}}
\newcommand{\Lwencdeltastaropt}[1]{\ensuremath{\mathcal{L}^{(#1)}_{\encopts{#1}{\alphabettwo}}(\pdeltastar)}}
\newcommand{\Lwencdeltaopt}[1]{\ensuremath{\mathcal{L}^{(#1)}_{\encopts{#1}{\alphabettwo}}(\pdelta)}}

\usepackage{framed}
\newenvironment{leftbar1}[1][\hsize]
{%
    \MakeFramed{\hsize#1\advance\hsize-\width\FrameRestore}%
}
{\endMakeFramed}

\definecolor{ethblue}{rgb}{0.0,0.0,0.0}
\newcommand{\ethletter}{
    \hspace{-0.5mm}\text{
    \fontfamily{phv}\fontseries{bx}\fontsize{7}{\baselineskip}\selectfont
    \textit{\textbf{\color{ethblue}{E}}}}
}
\definecolor{jhublue}{rgb}{0,0.1,0.4}
\newcommand{\jhuletter}{
    \hspace{-0.5mm}\text{
    \fontfamily{phv}\fontseries{bx}\fontsize{7}{\baselineskip}\selectfont
    \textbf{\color{jhublue}{J}}}
}
\newcommand{\hrefEmail}[2]{\href{mailto:#1}{\color{black}{#2}}}

\let\svthefootnote\thefootnote
\title{Tokenization and the Noiseless Channel}

\author{
Vilém Zouhar$^{\ethletter}$ \quad Clara Meister$^{\ethletter}$ \quad Juan Luis Gastaldi$^{\ethletter}$ \quad Li Du$^{\jhuletter}$ \\
\bf Mrinmaya Sachan$^{\ethletter}$ \quad Ryan Cotterell$^{\ethletter}$ \\
\text{} \\
ETH Z{\"u}rich$^{\ethletter}$ \quad Johns Hopkins University$^{\jhuletter}$ \\
\texttt{\{\hrefEmail{vzouhar@ethz.ch}{vzouhar},\hrefEmail{cmeister@ethz.ch}{cmeister},\hrefEmail{gjuan@ethz.ch}{gjuan},\hrefEmail{msachan@ethz.ch}{msachan},\hrefEmail{rcotterell@ethz.ch}{rcotterell}\}@ethz.ch}
\quad
\texttt{\hrefEmail{leodu@cs.jhu.edu}{leodu@cs.jhu.edu}}
}

\begin{document}
\maketitle

\begin{abstract}
Subword tokenization is a key part of many NLP pipelines.
However, little is known about why some tokenizer and hyperparameter combinations lead to better downstream model performance than others. 
We propose that good tokenizers lead to \emph{efficient} channel usage, where the channel is the means by which some input is conveyed to the model and efficiency can be quantified in information-theoretic terms as the ratio of the Shannon entropy to the maximum possible entropy of the
token
distribution.
Yet, an optimal encoding according to Shannon entropy assigns extremely long codes to low-frequency
tokens
and very short codes to high-frequency
tokens.
Defining efficiency in terms of Rényi entropy, on the other hand, penalizes distributions with either very high or very low-frequency 
tokens.
In machine translation, we find that across multiple tokenizers, the Rényi entropy with $\alpha = 2.5$ has a very strong correlation with \bleu: $0.78$ in comparison to just $-0.32$ for compressed length.\looseness=-1
\vspace{-12.5pt}
\end{abstract}

\let\thefootnote\relax\footnote{\hspace{-6mm}\raisebox{-1mm}{\includegraphics[width=4mm]{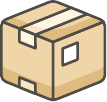}}\hspace{-0mm} We release the \texttt{\href{https://github.com/zouharvi/tokenization-scorer}{tokenization-scorer}} package (\Cref{sec:package_howto}).}
\addtocounter{footnote}{-1}\let\thefootnote\svthefootnote

\section{Introduction}
Tokenization, the practice of breaking up text into words or subword pieces, or more generally, \textit{tokens},\footnote{To avoid ambiguity, 
we eschew the common expressions \textit{word} and \textit{subword} and, instead, adopt the term \textit{token} to mean an element of the vocabulary \emph{after} tokenization.
We formally define token in \cref{sec:tokenization_compression}.} is often the first step in an NLP pipeline.
A wide variety of tokenization functions have been proposed in the NLP literature \cite{mermer2010unsupervised,sennrich2016,kudo2018subword}.
And, indeed, research on developing a good tokenization function continues because \emph{how} one tokenizes may have a large impact on model performance in the downstream task.
For instance, \citet{gowda2020finding} note \bleu ranges from 28 to 37 just by changing the size of the vocabulary in their machine translation (MT) pipeline. 
A direct \emph{extrinsic} evaluation of a tokenization function, however, is computationally intensive: One first has to retokenize the corpora (generally quick), but then retrain the NLP model (often computationally intensive) to evaluate the effect.
For this reason, characterizing the \emph{intrinsic} properties of a good tokenization function has practical benefits \citep{galle2019investigating,gowda2020finding}.\looseness=-1

Our paper takes an information-theoretic approach to characterizing a good tokenization function.
Following \citet{galle2019investigating}, we contend that tokenization may be fruitfully viewed as determining a good dictionary code for a language.
Fortunately, dictionary codes are equipped with a natural intrinsic metric of utility (expected code length) whereas there are many ways to extrinsically measure the tokenization quality.
For simplicity and in line with previous research, we choose a specific downstream task metric: \bleu \citep{papineni2002bleu} in the domain of MT.\footnote{In \Cref{tab:comparison_otherlangs} (Appendix) we replicate the findings with \chrf \citep{popovic2015chrf}, \bleurt \citep{sellam2020bleurt} and \comet \citep{rei2020comet}.}
 
We hypothesize that, \textit{ceteris paribus}, downstream task metrics should correlate with the expected code length of the unigram token distribution. 
While not immediately intuitive, the motivation is that  there is a theoretical connection between the expected code length (under an optimal encoder) of a token distribution and that distribution's  Shannon entropy: The latter gives us a lower bound on the former. And given a fixed vocabulary size, higher entropy token distributions are more desirable because they are more balanced, i.e., there are fewer tokens that occur too rarely or too frequently. 
This characteristic should in turn balance a model's ability to learn representations for the entire vocabulary, which requires exposure to enough instances of each token while also penalizing the use of very frequent character sequences as tokens, which is often inefficient due to their lack of distinct meaning.\looseness=-1

Yet when using Shannon entropy as our metric of a distribution's balance, the optimal token distribution may still include a large number of infrequent tokens. 
This behavior may be undesirable for a number of reasons that we subsequently discuss. 
Accordingly, we formulate the \defn{compression principle}, which states that downstream task metrics, e.g., \bleu, should correlate with the expected code length subject to a penalty for long codewords (which correspond to infrequent tokens).
Consequently, we introduce a more nuanced formulation of efficiency that employs Rényi entropy \citep{renyi1961measures},\footnote{\citet{campbell1965coding} shows that Rényi gives a lower-bound on the expected length of an optimal code subject to a length penalty, i.e., overly long or short codewords are penalized.} whose hyperparameter $\alpha$ allows us to penalize the use of long codes to varying degrees.\looseness=-1

In the experimental portion of our paper, we predict the performance of MT models. 
We find that the channel efficiency with Rényi entropy with $\alpha = 2.5$ yields a Pearson correlation of $0.78$ with \bleu on German $\rightarrow$ English MT (1M parallel sentences from CommonCrawl).
This stands in contrast to Shannon entropy or expected sequence length, which yield Pearson correlations of only $0.22$ and $-0.30$, respectively.

We also provide an easy-to-use package to score tokenizations.
See \Cref{sec:package_howto} for usage instructions.

\section{Tokenization}
\label{sec:tokenization_compression}

Tokenization is generally defined informally as the breaking up of text into a sequence of tokens which are then encoded into a machine-interpretable format.
However, to proceed with our analysis, we require a more formal treatment. 
First, we assume that there exists an alphabet, a finite, non-empty set of \defn{characters} $\alphabetone$.
We call a string of characters $\bsigma = \langle\sigma_1 \sigma_2\cdots \sigma_N\rangle \in \alphabetone^*$ a \defn{text}.
In this formulation, we assume that the alphabet $\alphabetone$ includes all characters, including punctuation and a distinguished white space character.
Finally, an unordered multiset of texts $\{ \bsigma_1, \ldots, \bsigma_M\} \subset \alphabetone^*$ is termed a \defn{corpus} of size $M$.
We denote the true distribution over all texts as $\psigmastar$.
Every $\psigmastar$ induces a marginal distribution over $\alphabetone$, which we call the 
$\alphabetone$-\defn{unigram distribution}:
\begin{equation}
\psigma(\character) \defeq \sum_{\stext \in \alphabetone^*} \psigmastar(\stext)\,\frac{\countc(\character, \stext)}{|\stext|}
\label{def:unigram_distribution}
\end{equation}
where $\countc(\character, \stext)$ returns the number of times the character $\character$ appears in text $\stext$.
In general, we do not have access to $\psigmastar$ but rather only to samples from $\psigmastar$ with which we can represent an empirical distribution $\psigmastarempirical$.
Our formal analysis, however, will consider $\psigmastar$.
 
Let $\alphabettwo$ be a second alphabet, which we call the \defn{tokenization alphabet}.
We define a \defn{tokenization function} $t: \alphabetone^* \rightarrow D \subseteq \alphabettwo^*$ as a function mapping texts in alphabet $\alphabetone$ to sequences of \defn{tokens} in $D =t(\alphabetone^*)$,
One popular choice in NLP is to have $\alphabetone$ be a set of Unicode characters and $\alphabettwo$ be a set of \emph{strings} of Unicode characters.
In this case, the tokenization function $t$ segments the text $\stext$ into tokens corresponding to smaller chunks of text.
There are many approaches for devising different $t$'s; a brief overview of some of them is offered in \Cref{sec:tokenizer_details}.\looseness=-1

Furthermore, for our purposes, it is useful to restrict tokenization functions to those that are invertible (bijections), i.e., rules where we can undo the tokenization.
This way, the original text can be reconstructed and no information is lost during tokenization.\looseness=-1

\begin{myexample}
\label{example:cows_text}
Any injective tokenization function, i.e., mapping different inputs $\bdelta', \bdelta'' $ to different outputs $\stext', \stext''$, satisfies our requirements.
As a counter example, consider a tokenization function $t_1$ for which $t_1(\texttt{two{\spacesymb}cows}) = \langle \texttt{two}, \spacesymbB, \texttt{[UNK]} \rangle$  and $t_1(\texttt{two{\spacesymb}birds}) = \langle \texttt{two}, \spacesymbB, \texttt{[UNK]} \rangle$.
The $t_1$'s lack of injectivity prevents us from recovering the original text from the token sequence $\langle \texttt{two}, \spacesymbB, \texttt{[UNK]} \rangle$.\looseness=-1

\end{myexample}

Because of our restriction to invertible tokenization functions, with a change of variable we can convert the distribution over texts in
$\alphabetone^*$ into one over token sequences $\bdelta$ in $D$ in a straightforward manner: $\pdeltastar(\bdelta) =  \psigmastar(t^{-1}(\bdelta))$.
Note that the pushforward $\pdeltastar(\bdelta)$ induces a distribution over $\alphabettwo^*$ but with support limited to $D$.\looseness=-1

In applied NLP, there is currently no widely accepted notion of the \textit{intrinsic} quality of a tokenization function. 
Rather, practitioners are generally interested in its \emph{extrinsic} performance, i.e., the performance of a model trained on a corpus tokenized using a certain tokenization function.
Under such an evaluation, given two tokenization functions, the one that enables better performance on the downstream task is taken to be better.
However, gauging the quality of a tokenizer function in this manner is computationally expensive.
Thus, we develop an information-theoretic intrinsic evaluation.\looseness=-1
\section{Communication in a Noiseless Channel}
\label{sec:first_compression_principle}

Our analysis of tokenization schemes relies on the following framing: Our ultimate goal when tokenizing a text $\stext \sim \psigmastar$ is the transmission of this text across a hypothetical channel.  
To perform this feat, we first tokenize $\stext$ into a sequence in $D \subseteq \alphabettwo^*$.
We then encode each token in $\alphabettwo$ as a sequence of \defn{symbols} from the set $\{1,\ldots, b\}$, where $b$ is determined by the channel. 
Our goal is to analyze the properties of tokenization schemes that lead to models with good downstream performance.

In the case of a noisy channel, we seek an encoding scheme that will help ensure that $\bsigma$ is resilient to noise in addition to efficiently encoding $\bsigma$.
However, in the noiseless case, we \emph{only} care about efficiency. 
We can assume that we are working with a noiseless channel because, in the process of encoding data, no information is ever altered by a stochastic process.
In this case, one can equivalently think of noiseless channel encoding as compression.
Thus, our analysis proceeds by considering the efficiency of different tokenization functions as if our goal is to use them to communicate over a noiseless channel.
To this end, we first discuss the conditions for building such an encoding and then discuss the concept of efficient channel usage.\looseness=-1

\begin{definition}
A \defn{token-level encoder} $\enc{\alphabettwo}$ is a function $\enc{\alphabettwo} : \alphabettwo \rightarrow \{1,\ldots, b\}^*$ that maps every token $\delta \in \alphabettwo$ to a \defn{string} of symbols in base $b$, which we call a codeword.
We can naturally lift the token-level encoder to a \defn{sequence-level encoder} using concatenation as $\enc{\alphabettwo}(\ssymbols) = \bigoplus_{n = 1}^{|\ssymbols|} \enc{\alphabettwo}(\ssymbol_n)$.\footnote{Whitespace information is not lost here because it is included in the tokens (see \Cref{example:cows_text}).}
\end{definition}

In order to be able to uniquely decode a string $\ssymbols$, we further require that $\enc{\alphabettwo}$ produces prefix-free\footnote{Prefix-free means that no codeword is a prefix of any other codeword.} codes for all tokens in $\alphabettwo$.
As an example, Huffman encoding provides a fast and nearly optimal (in a sense to be discussed in the subsequent section) method to construct prefix-free codes \citep{huffman1952method}.\looseness=-1

\begin{myexample}[One-hot encoding]
Consider a tokenization alphabet $\alphabettwo$. 
In NLP, when $b=2$, the most straightforward way of encoding the $n^{\text{th}}$ element of $\alphabettwo$ is a vector of zeroes with length $|\alphabettwo|$ with $1$ on position $n$.
\end{myexample}

\begin{myexample}[Transmission]
\label{ex:cow_encoding}
\newcommand{\colorCowA}{PaleGreen4}
\newcommand{\colorCowB}{SteelBlue4}
\newcommand{\colorCowC}{DeepPink4}

We consider an arbitrary encoder $\enc{{\alphabettwo}}$ over a given alphabet $\alphabettwo$ and a channel with $b=2$.
Now given a text and some tokenization function $t(\texttt{two{\spacesymbB}cows})$ $= \langle \texttt{two\spacesymb}, \texttt{cow}, \texttt{s} \rangle$ we apply the encoder: $\enc{{\alphabettwo}}(\texttt{two\spacesymb}) = \textcolor{\colorCowA}{\texttt{1010101}}$ $\enc{{\alphabettwo}}(\texttt{cow})$ $= \textcolor{\colorCowB}{\texttt{101111101}}$ and $\enc{{\alphabettwo}}(\texttt{s})$ $= \textcolor{\colorCowC}{\texttt{01010}}$ and for the whole
sequence
$\enc{{\alphabettwo}}(t(\texttt{two{\spacesymbB}cows}))$ $= \textcolor{\colorCowA}{\texttt{1010101}}\textcolor{\colorCowB}{\texttt{101111101}}\textcolor{\colorCowC}{\texttt{01010}}$.
\end{myexample}

For the remainder of the paper, we will not be interested in any specific $\enc{\alphabettwo}$, but rather in an \emph{optimal} token-level encoder that we can achieve, as measured by expected code length measures.\looseness=-1
\begin{definition}
The \defn{expected code length} $\mathcal{L}_\enc{\alphabettwo}$ of a token-level encoder $\enc{\alphabettwo}$ is defined as
\begin{align}\label{eq:expected-code-length}
\Lwenconedelta = \sum_{\delta \in \alphabettwo} \pdelta(\delta) |\enc{\alphabettwo}(\delta)|
\end{align}
\end{definition}
A well-known result from information theory tells us that \cref{eq:expected-code-length} is bounded by the Shannon entropy of $\rvDelta$, a $\alphabettwo$-valued random variable with law $\pdelta$.
To introduce the theorem, we first define Shannon entropy.\looseness=-1

\begin{definition}
The \defn{Shannon entropy} of $\rvDelta$ is defined as\looseness=-1
\begin{equation}
\ent(\rvDelta) \defeq -\sum_{\delta \in \alphabettwo} \pdelta(\delta) \log \pdelta(\delta)
\end{equation}
For channels using $b$ symbols for transmission, the logarithm is of base $b$.
Traditionally in information theory, ones takes $b=2$.\looseness=-1
\end{definition}

\begin{theorem}[Shannon's Source Coding
Theorem]
\label{thm:shannon_source_coding}
Let $\rvDelta$ be a $\alphabettwo$-valued random variable with law $\pdelta$ and let $\enc{\alphabettwo}$ be an encoder.
Then, 
\begin{equation}
    \ent(\rvDelta) \leq \Lwenconedelta
\end{equation}
with the optimal token-level encoder $\encopt{\alphabettwo}$ satisfying\looseness=-1
\begin{equation}
     \Lwenconedeltaopt \leq \ceil*{\ent(\rvDelta) }
\end{equation}
\end{theorem}
This theorem tells us that if we wish to communicate tokens from the alphabet $\alphabettwo$ through a noiseless channel, their minimum expected length for any possible encoding is bounded by the Shannon entropy of the distribution $\pdelta$. 
An optimal token-level encoder will produce codes with the expected length within those exact bounds. 
We can prove a very similar result to Shannon's source coding theorem \citep{shannon1948mathematical} that tells us how well we can optimally encode a $\alphabettwo^*$-valued source using only the token-level encoder.\looseness=-1

To this end, we first introduce the notions of expected sequence length and average per-token encoding length, and then offer a lower-bound on the compression achievable using only a token-level encoder.
We additionally define three random variables that will prove useful in our analysis; all of them are pushforwards of $\pdeltastar$.

Let $\strlen$ be a random variable whose values range over strings' lengths, i.e., $L(\bdelta) = |\bdelta|$. 
The expected token sequence length $\Lexpect$ for sequences sampled according to $\pdeltastar$ is then\looseness=-1 
\begin{gather}
\Lexpect  =\sum_{\bdelta \in \alphabettwo^*} \pdeltastar(\bdelta) |\bdelta|
\end{gather}
where for notational simplicity, we leave the dependence of this expectation on $\pdeltastar$ implicit as it will always be clear from context.
Let $X_{\delta}(\bdelta) = \frac{\countc(\delta, \bdelta)}{|\bdelta|}$ be the unigram random variable, i.e., a function that returns the proportion of $\bdelta$ that consists of a particular $\delta$. 
Finally, define the random variable $\avgencodinglen(\bdelta) = \sum_{\delta \in \alphabettwo} X_{\delta}(\bdelta) |\enc{\alphabettwo}(\delta)|$.

We now turn to our first major theorem.
\begin{restatable}{theorem}{codebounded}
\label{th:codebounded}
Let $\pdeltastar$ be a distribution over $\alphabettwo^*$, and let $\pdelta$ be the unigram distribution induced by $\pdeltastar$ (\cref{def:unigram_distribution}).
Then, for an optimal token-level encoder $\encopt{\alphabettwo} : \alphabettwo \rightarrow \{1,\ldots, b\}^*$ lifted to the sequence level, the following lower and upper bounds hold:\footnote{Be careful not to confuse $\encopts{\star}{\alphabettwo}$, an optimal token-level encoder (here lifted to the sequence level), with $\enc{\alphabettwo^*}$, an arbitrary sequence-level encoder, which we usually denote $\encn$ when clear from context.}
\begin{flalign}
&\begin{aligned}
\ent(\rvDelta) &\leq \frac{\Lwenconeopt - \mathrm{Cov}(\avgencodinglenopt,\strlen)}{\Lexpect}  \\
&\leq \lceil \ent(\rvDelta) \rceil
\end{aligned}
&&
\end{flalign}
\end{restatable}
\begin{proof}
The proof is given in \Cref{sec:proofs}.
\end{proof}

In the special case of Shannon entropy, we additionally arrive at the following stronger inequality:
\begin{equation}
    \Lwenconeoptstar \leq  \Lwenconeopt
\end{equation}
This holds because $\encoptn$ is \emph{not} constrained to token-level codes and the unconstrained minimum over all codes is naturally lower than the constrained version. 
As a concrete example, even if two $\delta', \delta'' \in \alphabettwo$ \emph{always} appear together in practice, $\encopt{\Delta}$ must assign both $\delta'$ and $\delta''$ their own unique code.
Such a constraint does not apply to $\encoptn$.
We foreshadow that this inequality does \emph{not} generalize to R{\'e}nyi entropy, as discussed in \cref{sec:renyi}.\looseness=-1

\Cref{th:codebounded} tells us that the expected code length of a sequence-level encoder, based on a token-level encoder, is proportional to the expected code length of the unigram distribution up to an additive covariance factor.
This allows us to determine both a lower-bound for the expected code length of such an encoder and
an upper-bound for the expected code length of a sequence-level encoder based on an optimal token-level encoder.\looseness=-1

We are now in the position to return to the main objective of this paper: Assessing the quality of different tokenizers. 
One natural way of comparing tokenizers would be to compare properties of the distributions over tokens that they each induce. 
At first glance, Shannon entropy looks like the most obvious candidate for such a property. 
However, for distributions over $\alphabettwo$ of different sizes, it is not directly comparable. The efficiency of a tokenization function addresses this issue. \looseness=-1
\begin{definition}
\label{def:channel_efficiency}
Let $\psigmastar$ be a distribution over $\alphabetone^*$, let $t : \alphabetone^* \rightarrow \alphabettwo^*$ be a tokenization function, and let $\pdeltastar$ be the distribution over $\alphabettwo^*$ induced by $t$.
The \defn{efficiency} of $t$ is defined as
\begin{equation}
    \eff(\psigmastar,t) \defeq \frac{\Lwenconeopt}{\mathcal{L}_{\encuni{\alphabettwo}}(\pdeltastar)}
\end{equation}
where $\encopt{\alphabettwo}$ is an optimal token-level encoder and $\encuni{\alphabettwo}$ a uniform encoder that assigns all tokens in $\alphabettwo$ codes of equal length: $\lceil\log |\alphabettwo|\rceil$.

\begin{restatable}{theorem}{efficiencybound}
\label{thm:efficiencybound}
The efficiency of $t$ is upper-bounded by
\begin{equation}\label{eq:efficiency-upper-bound}
\frac{\lceil \ent(\rvDelta)\rceil +\frac{\mathrm{Cov}\left( \avgencodinglenopt, \strlen \right)}{\Lexpect}}{\log |\alphabettwo|} \geq \eff(\psigmastar, t) 
\end{equation}
and lower-bounded by
\begin{equation}
\frac{\ent(\rvDelta) +\frac{\mathrm{Cov}\left( \avgencodinglenopt, \strlen \right)}{\Lexpect}}{\lceil\log |\alphabettwo|\rceil} \leq \eff(\psigmastar, t) 
\end{equation}
where $\rvDelta$ is a $\alphabettwo$-valued random variable with law $\pdelta$, the unigram distribution induced by $t$.
\end{restatable}
\begin{proof}
The proof is given in \Cref{sec:proofs}.
\end{proof}

\end{definition}
Note that the upper bound given in \cref{eq:efficiency-upper-bound} tells us how efficient the \emph{best} code could be, which is the more interesting bound for our purposes.
Additionally, we note that, by introducing a normalization factor, efficiency provides a better solution than directly comparing distributions' entropies. We illustrate this in the following example. \looseness=-1

\begin{myexample}\label{ex:entropy-vs-efficiency}
Consider a tokenization function $t_1$ with tokenization alphabet $\alphabettwo_1$ where $|\alphabettwo_1|=6$.
We then introduce a second tokenization function $t_2$ with a tokenization alphabet $\alphabettwo_2$ defined to be $\alphabettwo_1$ plus an additional $6$ tokens that occur very infrequently.
The difference between these two distributions is illustrated in \Cref{fig:inefficient_uniform_distributions}.
If, for example, we relied solely on the Shannon entropy, which is higher for more uniformly spread-out distributions, we would judge the second distribution to be better ($2.50 < 3.08$).
However, the efficiency tells the opposite story ($0.97\% > 0.86\%$).

\end{myexample}

As \Cref{ex:entropy-vs-efficiency} suggests, the measure provided by efficiency is in line with the idea of a more balanced distribution over $\alphabettwo^*$.
Informally, we do not want a tokenizer that induces a distribution with very low entropy, as this is indicative of an unbalanced distribution. 
The efficiency $\eff$ provides us with a notion of this imbalance. 
To relate efficiency back to our metaphor of the noiseless channel, we note that the quantity $1-\eff(\pdeltastar, t)$ is known as \defn{relative redundancy} and corresponds to the maximum data compression ratio (in percentage of how much can data size be reduced) that can be achieved.\looseness=-1

\section{Rényi Efficiency}\label{sec:renyi}

\Cref{def:channel_efficiency}, the standard definition of efficiency, is based on Shannon entropy.
Upon closer inspection, we see it linearly penalizes the use of long codes.
To see why, consider a case where the distribution changes such that the entropy increases by one.
Then, the upper-bound for the expected code length provided by an optimal encoder also increases by one.
However, in some cases, we may wish to assign a non-linear cost to code length, e.g., there may be a non-linearly higher cost for decoding longer codes.
In the context of choosing the vocabulary for a model, this corresponds to our desire to avoid inducing tokens that occur very infrequently because there may not be enough examples of them in the training data for the model to learn.
To add an additional degree of freedom to accommodate such preferences \citet{campbell1965coding} generalizes the measure of expected code length to \defn{discounted expected code length} for a hyperparameter $s$ as follows:\footnote{Our notation differs slightly from \citet{campbell1965coding}.}
\begin{equation}\label{eq:discounted-length}
\Lwencdelta{s} \defeq \lim_{s' \rightarrow s} \frac{ \log\left(\sum_{\delta \in \alphabettwo} \pdelta(\delta) b^{s'|\enc{\alphabettwo}(\delta)|}\right)}{s'}
\end{equation}
By L'Hôpital's rule, we can show that
\begin{equation}
    \Lwencdelta{0} = \sum_{\delta \in \alphabettwo} \pdelta(\delta) |\enc{\alphabettwo}(\delta)|
\end{equation}
and, additionally, that
\begin{equation}
    \Lwencdelta{\infty} = \max_{\delta \in \alphabettwo} |\enc{\alphabettwo}(\delta)|
\end{equation}
Beyond the limiting cases, for $s \in (-1, \infty) \setminus \{0\}$, we further note that $\Lwencdelta{s}$ is a monotonically increasing function of $s$.
The larger our value of $s$, the more disproportionately $\Lwencdelta{s}$ increases as a function of the \emph{longest} codeword, which often corresponds to the encoding of a low-frequency character in a good code because high-frequency tokens are assigned the shorter codes in order to minimize the expected code length. 
For large enough $s$, this has the effect of encouraging all codewords to be roughly of equal length.
\citet{campbell1965coding} sought an analogue of Shannon's coding theorem for $\Lwencdelta{s}$ where $s \neq 0$.
As it turns out, there is a deep connection with the Rényi entropy.\looseness=-1

\begin{figure}
\centering
\hspace{-2mm}
\includegraphics[width=0.54\linewidth]{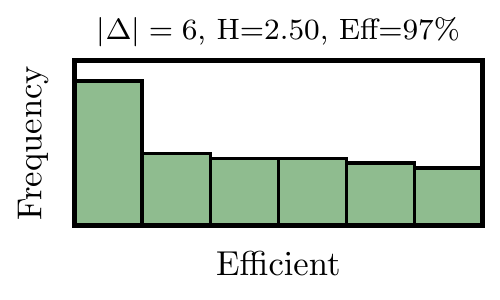}
\includegraphics[width=0.46\linewidth,trim=7.5mm 0 0 0,clip]{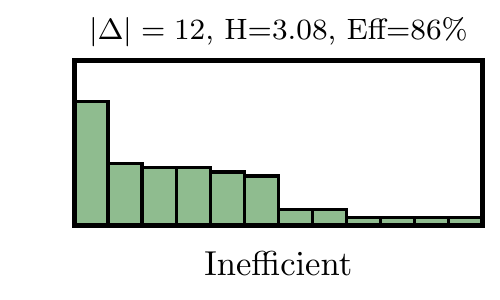}
\vspace{-3mm}
\caption{Examples of unigram distributions with efficient and inefficient channel usage.}
\label{fig:inefficient_uniform_distributions}
\end{figure}

\begin{definition}
The \defn{Rényi entropy} of order $\alpha > 0$ is defined as
\begin{equation}
\ent_\alpha(\pdelta) =\! \lim_{\alpha'\rightarrow \alpha} \frac{1}{1-\alpha'}\log \left( \sum_{\delta \in \alphabettwo} \pdelta(\delta)^{\alpha'} \right)
\end{equation}
\end{definition}
Prima facie, Rényi entropy bears some semblance to $\Lwencdelta{s}$.
To see this, consider the limiting cases.
At $\alpha = 0$, we have\looseness=-1
\begin{equation}
\ent_0(\pdelta) = \log |\alphabettwo|
\end{equation}
And, at $\alpha = \infty$, we have
\begin{equation}
\ent_\infty(\pdelta) = \max_{\delta \in \alphabettwo} -\log \pdelta(\delta)
\end{equation}
Finally, we have $\ent_1(p) = \ent(p)$, i.e., $\alpha=1$ corresponds to Shannon entropy, another result which can be shown by L'Hôpital's rule. 
These examples suggest the correspondence $\alpha = (1 + s)^{-1}$, which fits the three cases considered, e.g., note that $\alpha = 0$ when $s \rightarrow \infty$.
Moreover, this is exactly the intuition we argued for above:
When $\alpha = 0$, we encode tokens with codewords of the same length which follows from minimizing the length of the longest codeword.
On the other hand, when $s = -1$, we encourage shorter codes for high-probability tokens.
This case corresponds to $\alpha = \infty$.
We now prove that, similarly to how $\ent(\rvDelta)$ provides bounds for $\Lwenconeopt$ in \Cref{th:codebounded},
$\ent_\alpha(\rvDelta)$ provides bounds for
$ \Lwencdeltastaropt{s}$, where $\encopts{s}{\alphabettwo}$ is an encoder optimal with respect to a given $s = \alpha^{-1} + 1$.
We term such an encoder $s$-optimal.\looseness=-1

\begin{restatable}[Generalization of \citet{campbell1965coding}]{theorem}{generalizedcampbell}
\label{thm:generalized-campbell}
Let $\ent_\alpha$ be the Rényi entropy of order $\alpha$ and
let $\Lwencdelta{s}$ (\cref{eq:discounted-length}) be the discounted expected code length for the encoder $\enc{\alphabettwo}$, where $s = \alpha^{-1} - 1$.
Moreover, let $\rvDelta$ be a $\alphabettwo$-valued random variable with law $\pdelta$. 
Then for an $s$-optimal token-level encoder $\encopts{s}{\alphabettwo}$, the following bound holds on the discounted expected code length: 
\begin{equation}
    \ent_\alpha(\rvDelta) \leq \Lwencdeltaopt{s} \leq \lceil\ent_\alpha(\rvDelta)\rceil
\end{equation}
\end{restatable}
\begin{proof}
Proof in \Cref{sec:proofs}.
\end{proof}
Note that we have further generalized \citeposs{campbell1965coding} result by allowing some negative values for $s$, namely, $s>-1$.
As a result, we can induce additional non-linear weight on too \emph{short} codes as opposed to only \emph{long} codes.

Now we generalize the efficiency with respect to Shannon entropy to Rényi entropy.
Let $\encopts{s}{\alphabettwo}$ be an $s$-optimal token-level encoder over token alphabet $\alphabettwo$. %
Note that several terms from our prior notation can now be expressed in terms of $\encopts{s}{\alphabettwo}$, i.e.,
$\encopt{\alphabettwo} = \encopts{0}{\alphabettwo}$ and $\encuni{\alphabettwo} = \encopts{\infty}{\alphabettwo}$.
\begin{restatable}{theorem}{penalizedcodbounded}
\label{thm:penalizedcodbounded}
Let $\alpha = (1 + s)^{-1}$ and $\pdeltastar$ be a distribution over $\alphabettwo^*$, and let $\pdelta$ be the unigram distribution induced by $\pdeltastar$ (\cref{def:unigram_distribution}).
Then, the following inequality holds
\begin{align}
\lceil \ent_\alpha(\rvDelta) \rceil &\geq \frac{\Lwenctwo{s} - \mathrm{Cov}(\avgencodinglenopts{s},\strlen)}{\Lexpect} 
\end{align}
for 
an $s$-optimal sequence-level encoder $\encopts{s}{\alphabettwo}$ based on token-level encoder $\enc{\alphabettwo} : \alphabettwo \rightarrow \{1,\ldots, b\}^*$
.\looseness=-1

\end{restatable}
\begin{proof}
Proof in \Cref{sec:proofs}.
\end{proof}
\begin{definition}\label{def:eff_alpha}
Let $\psigmastar$ be a distribution over $\alphabetone^*$, let $t : \alphabetone^* \rightarrow \alphabettwo^*$ be a tokenization function, and let $\pdeltastar$ be the distribution over $\alphabettwo^*$ induced by $t$.
The \defn{Rényi efficiency} of $t$ at $\alpha$ is defined as
\begin{subequations}
\begin{align}
    \eff_\alpha(\psigmastar, t) &\defeq \frac{\Lwenctwo{s}}{\Lwenctwo{\infty}} \\
    &= \frac{\Lwenctwo{s}}{\mathcal{L}_{\encuni{\alphabettwo}}(\pdeltastar)}
\end{align}
\end{subequations}
where $s=\alpha^{-1} - 1$.
\end{definition}
The R{\'e}nyi efficiency can be easily upper-bounded in a similar manner to the Shannon efficiency.
\begin{restatable}{theorem}{renyiefflower}
\label{thm:renyi_eff_lower}
Let $\psigmastar$ be a distribution over $\alphabetone^*$, let $t : \alphabetone^* \rightarrow \alphabettwo^*$ be a tokenization function, and let $\pdeltastar$ be the distribution over $\alphabettwo^*$ induced by $t$.
Then, for an $s$-optimal token-level encoder $\encopts{s}{\alphabettwo}$ lifted to the sequence-level,
the Rényi efficiency of $t$ at $\alpha$ is upper-bounded by\looseness=-1
\begin{equation}
\frac{\lceil \ent_\alpha(\rvDelta) \rceil +\frac{\mathrm{Cov}\left( \avgencodinglenopts{s},\, \strlen \right)}{\Lexpect}}{ \log |\alphabettwo| } \geq \eff_\alpha(\psigmastar, t) 
\end{equation}
where $\rvDelta$ is a $\alphabettwo$-valued random variable with law $\pdelta$, the unigram distribution induced by $t$.
\end{restatable}
\begin{proof}
Proof in \Cref{sec:proofs}.
\end{proof}
To provide more intuition of why the non-linear penalization in Rényi efficiency makes for a good measure of distribution balance, we offer a worked example in \Cref{ex:peak_no_peak}.\looseness=-1

\begin{figure*}
\centering
\includegraphics[width=0.48\linewidth]{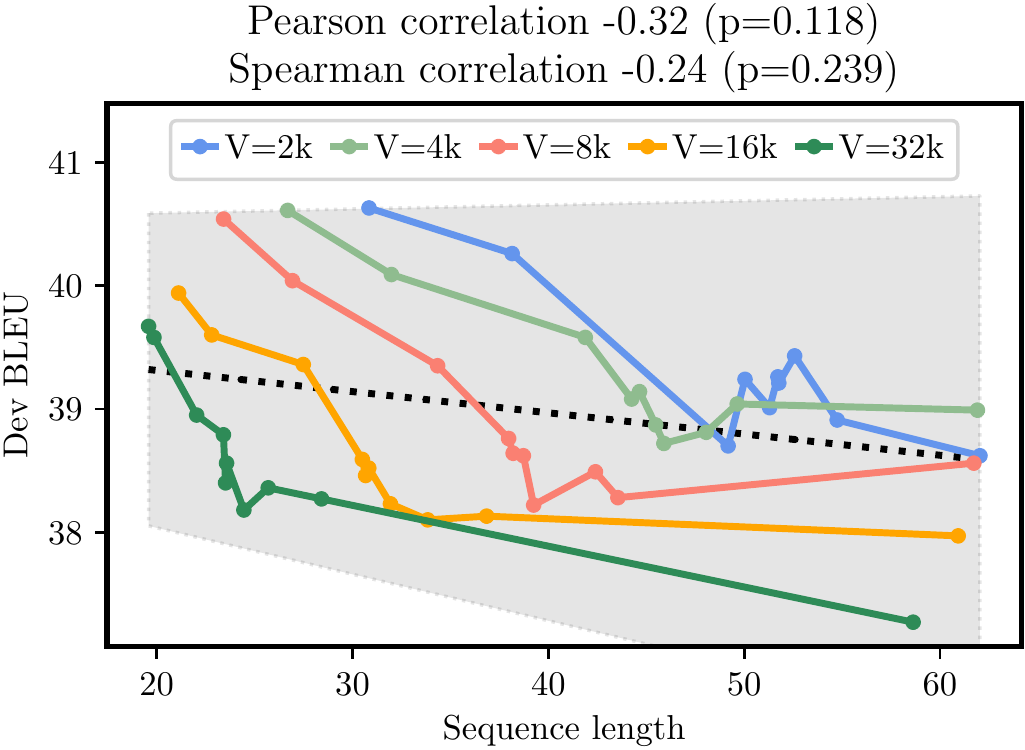}
\hspace{1mm}
\includegraphics[width=0.48\linewidth]{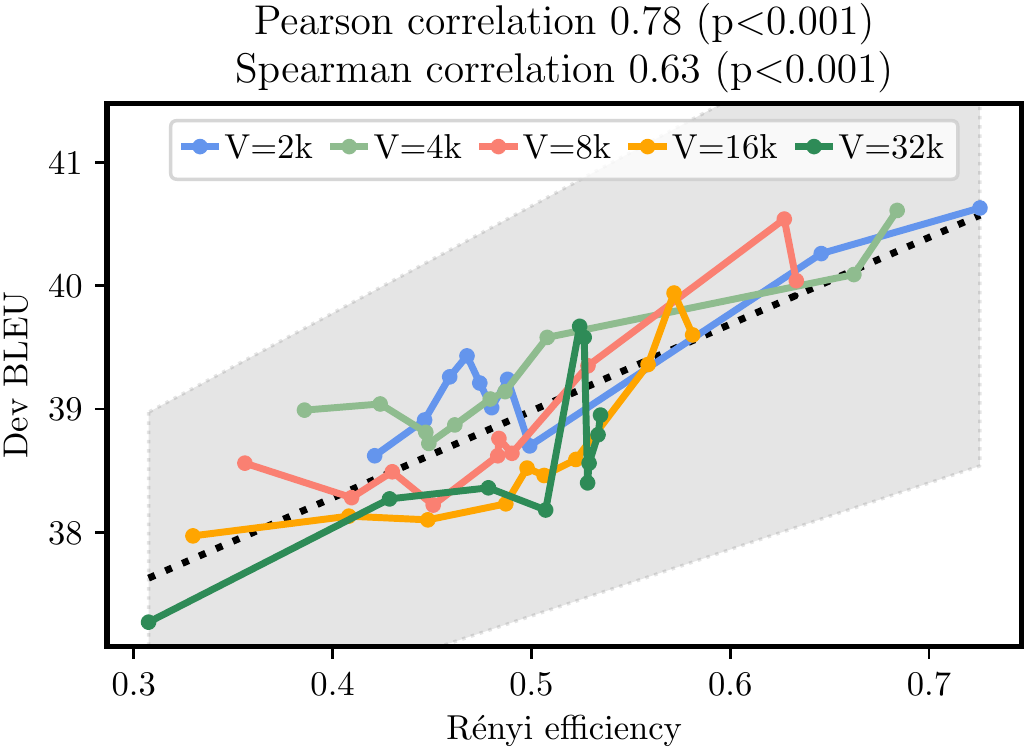}
\vspace{-3mm}
\caption{Efficiency of sequence length and $\ent_{2.5}/\ent_0$ as predictors of MT performance (average of 5 runs). Bands show 95\% $t$-test confidence intervals for regression lines.}
\label{fig:corr_renyi_eff}
\label{fig:corr_seq_len}
\end{figure*}
\section{The Compression Principle}
\label{sec:compression_principle}

In previous sections, we discussed how different tokenizers lead to token distributions of varying properties.
Now, we add the last piece necessary to link the downstream performance of a system with the choice of a tokenizer.

\begin{hypothesis}[Compression Principle]
\label{hyp:compression_principle}

Let $\psigmastar$ be a distribution over texts with characters from alphabet $\alphabetone$ and $t$ be a tokenization function from $\alphabetone^*$ to $\alphabettwo^*$ for some token alphabet $\alphabettwo$.
Let $\pdelta$ be the $\alphabettwo$-unigram distribution induced by $t$.
Finally, let $\performance(t)$ be some measure of performance of a system $M$ which uses tokenization $t$.
Then, for some $\alpha$ dependent on $M$, $\eff_\alpha(\psigmastar, t)$ is a good predictor of $\performance(t)$.\looseness=-1
\end{hypothesis}
In words, we hypothesize that the efficiency of the tokenization function $t$ is highly correlated with the downstream performance.
We will verify this claim experimentally in \Cref{sec:empirical_evidence}.\looseness=-1

\paragraph{Rényi Entropy $\alpha$.}
The choice of $\alpha$ for $\ent_\alpha$ determines the extent to which longer codewords are penalized.
On one hand, if we observe that Rényi efficiency with low $\alpha$ correlates the best with performance, we can conclude that longer codewords and, hence, very low-frequency tokens hurt performance.
On the other hand, if we observe that Rényi efficiency with high $\alpha$, we can conclude that shorter codewords and, hence, very high-frequency tokens hurt performance. 
Note that most downstream NLP applications do not explicitly use codewords (all token representations are the same size), but long codewords are still a natural way to think about low- and high-frequency tokens.\looseness=-1

\paragraph{Learnability.}

The most intuitive explanation for why some tokenization functions enable good downstream results and some worse is that having many low-frequent tokens will prevent the model from learning their distributional properties.
This hypothesis can be related back to the sample complexity of the learning algorithm, i.e., the number of training samples needed by the model in the given setting to learn the function of interest.
If we accept that part of the MT task is learning the meaning of all individual vocabulary tokens, then sample complexity could (at least partially) be expressed in terms of the number of instances of each token.
This argument is made by \citet{gowda2020finding}, who are concerned with what proportion of $\delta\in\alphabettwo$ appears at least 100 times in the corpus for the downstream task at hand. %

Nevertheless, we will see shortly that the best predictor with Rényi efficiency is
for
$\alpha > 1$, meaning that higher weight is given to codewords for more frequent tokens.
We therefore hypothesize, that very high-frequency tokens have the most impact in downstream performance.

\section{Experiments}
\label{sec:empirical_evidence}
We now seek empirical evidence for \Cref{hyp:compression_principle}.
We focus on MT, where a standard automatic evaluation metric is \bleu \citep{papineni2002bleu}.  We use the English$\rightarrow$German CommonCrawl dataset in all experiments. The specifics of the MT system, data and evaluation are described in \Cref{sec:mt_details}.
We consider two different experimental manipulations.
First, we experiment with various modifications of the popular byte-pair encoding (BPE) tokenizer \citep{sennrich2016} to control its compression rate.
The details are discussed in \Cref{sec:experiment_1}.
Second, we experiment with a variety of tokenization schemes: Unigram \citep{kudo2018subword}, WordPiece \citep{devlin2019}, Lempel--Ziv--Welch \citep{ziv1977universal,welch1984} and Morfessor \citep{creutz2007unsupervised,virpioja2013morfessor,smit2014morfessor}.
The details are discussed in \Cref{sec:experiment_2}. \looseness=-1

Note that throughout our experiments, we make the simplifying assumption of $\mathrm{Cov}\left( \avgencodinglenopts{s}, \strlen \right) = 0$.
It simplifies the upper bound of $\eff(\psigmastar, t)$ (from \Cref{thm:efficiencybound}) to $\frac{\lceil \ent(\rvDelta)\rceil}{\log |\alphabettwo|}$ and the upper bound of $\eff_\alpha(\psigmastar, t)$ (from \Cref{thm:renyi_eff_lower}) to $\frac{\lceil \ent_\alpha(\rvDelta) \rceil}{\log |\alphabettwo|}$.
From our preliminary results, $\mathrm{Cov}\left( \avgencodinglenopts{s}, \strlen \right)$
is negative and small.
We leave its more accurate approximation, which requires a Rényi analogue of Huffman coding as in \citet{jelinek-renyi-coding}, to future work.\looseness=-1

\subsection{Experiment 1}\label{sec:experiment_1}

In our first experiment, we analyze how predictive various quantitative attributes of a tokenization scheme are of downstream model performance.
We consider BPE with 5 different vocabulary sizes: 2k, 4k, 8k, 16k, and 32k.
For each vocabulary size, we create multiple tokenization schemes with varying compression rates.
As discussed in \Cref{subsec:bpe}, BPE produces a vocabulary through a greedy compression algorithm.
However, in order to achieve a variety of different compression rates, we inject random noise into the algorithm.\footnote{The empirical effects of choosing  BPE merge operations uniformly at random is studied by \citet{saleva2023what}.} 
We achieve this by sampling from a Boltzmann distribution over the pair frequencies with temperature parameter $\tau$; see \Cref{subsec:bpe} for details.\footnote{Note that the stochasticity in our case is introduced during training and not during inference, as in the popular BPE-dropout method \citep{provilkov2020bpe}.} We then treat each vocabulary size--temperature pair as a single data point in our analysis.\looseness=-1

Our main quantitative attribute of interest, i.e., predictor, is Rényi efficiency. 
Aside from Rényi efficiency, we further consider Shannon and Rényi entropies, Shannon efficiency, and average tokenized sequence length.
Further, one popular heuristics for choosing the vocabulary size is given and justified by \citet{gowda2020finding}.
It can be summarized as: ``\textit{Use the highest possible vocabulary size such that 95\% of [tokens] occur at least 100 times in the data.}''
While the constants seem arbitrary, this rule of thumb works well in practice \citep{gowda2022checks,dramko2022dire,kumar2022bpe}.
Nevertheless, it is stated in an algorithmic manner and not as a predictor of performance or learnability.
We attempt to turn it into a regressive predictor so as to make it more comparable with the other quantities studied.
Given $\pdelta$, let $f_n(\pdelta)$ symbolize the frequency of the $n^{\text{th}}$ percentile.
We then define the quantity
$F_{\gamma_1,\gamma_2}(\pdelta) = \sum_{\gamma_1 \leq n \leq \gamma_2 } f_n(\pdelta)$, which in words, is the sum of token frequencies from the $\gamma_1^\text{th}$ to $\gamma_2^\text{th}$ percentile.
The original work suggests examining the frequency of the 95$^\text{th}$ percentile, i.e., $\gamma_1 = \gamma_2 = 0.95$.
In contrast, we add an additional degree of freedom as we do not inspect a single percentile frequency but rather a sum across an interval. 
Later, we scan the whole space for $\gamma_1$ and $\gamma_2$ and show that there are better choices that lead to much higher correlations. 

\begin{table}
\centering
\resizebox{\linewidth}{!}{
\addtolength{\tabcolsep}{-0.5mm}
\begin{tabular}{l<{\hspace{-2mm}}rrr}
\toprule
\textbf{Predictor} &
\textbf{Pearson}\hspace{3mm} & \textbf{Spearman}\hspace{1mm} & ${\bm{\rho^2}}$\\
\midrule
Sequence len. &
$-0.32 \,{\scriptstyle(=0.118)}$ & $-0.24 \,{\scriptstyle(=0.239)}$ & 10\% \\
Percentile freq. &
\colorbox{\algcolorC}{$0.76$}\hspace{-1mm} ${\scriptstyle(<0.001)}$ & ${0.63} \,{\scriptstyle(<0.001)}$ & 58\% \\
Entropy &
$0.22 \,{\scriptstyle(=0.281)}$ & $0.12 \,{\scriptstyle(=0.578)}$ & 5\% \\
Entropy eff. &
$0.56 \,{\scriptstyle(=0.004)}$ & $0.38 \,{\scriptstyle(=0.006)}$ & 31\%\\
Rényi entropy &
$0.49 \,{\scriptstyle(=0.001)}$ & $0.38 \,{\scriptstyle(=0.006)}$ & 24\%\\
Rényi eff. & 
\colorbox{\algcolorB}{$0.78$}\hspace{-1mm} ${\scriptstyle(<0.001)}$ & ${0.66} \,{\scriptstyle(<0.001)}$ & 61\% \\
\bottomrule
\end{tabular}
}
\caption{Correlations between different predictors and MT performance (\bleu). 
The $p$-values for each statistic (computed using a $t$-test) are in parentheses.}\label{tab:comparison_corr_all}
\vspace{-.5cm}
\end{table}

We use Pearson and Spearman correlations with downstream model performance (measured with \bleu) as our metrics of predictor quality. 
Recall that Pearson correlation tells us the strength of a \emph{linear} relationship between two variables.
On the other hand, Spearman correlation quantifies the strength of a linear relationship of the \emph{ranking}.\looseness=-1  

\begin{figure}
\centering
\includegraphics[width=\linewidth]{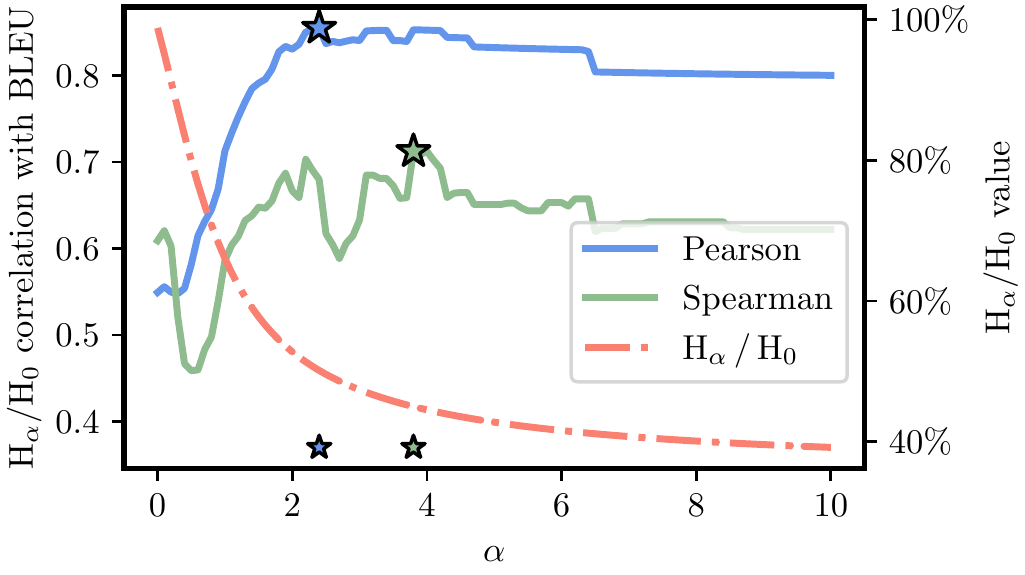}
\vspace{-9mm}
\caption{Correlation of Rényi efficieny ($\ent_\alpha/\ent_0$) with \bleu on train data in the experiment 1. Maximums of Pearson and Spearman correlations are marked with $\star$.}
\label{fig:corr_renyi_alpha}
\vspace{-.5cm}
\end{figure}

\paragraph{Results.}  In order to select $\alpha$ (for $\eff_\alpha$) as well as $\gamma_1$ and $\gamma_2$ ( for $F_{\gamma_1,\gamma_2}$), we use half of the data to perform a grid-search, selecting the hyperparameters that lead to the highest Pearson correlation.
We show the results of this grid search for $\ent_\alpha/\ent_0$ in \Cref{fig:corr_renyi_alpha} ($\alpha^* \doteq 2.5$) and for $F_{\gamma_1,\gamma_2}$ in \Cref{fig:freq_alphas_grid} ($\gamma_1^* \doteq 0.03, \gamma_2^* \doteq 0.83$).
Unless otherwise stated, we use these values in subsequent experiments.
We show the relationship between \bleu, sequence length and  Rényi efficiency as approximated by the lower bound (\Cref{thm:renyi_eff_lower}) in \Cref{fig:corr_renyi_eff}. A comprehensive comparison for all predictors is shown in \Cref{tab:comparison_corr_all}.
The visualization of the other predictors is in \Cref{fig:other_predictor_visualization}.
From these analyses, we can see that the Rényi efficiency provides a significantly better explanation for downstream model performance than any of our other predictors. 

When examining  which $\alpha$ leads to the highest absolute correlation with \bleu, we can conclude that tokenization schemes that result in fewer very high-frequency tokens are the best for downstream performance. 
This is evinced by both the relatively high value of $\alpha$ that leads to the best correlation with performance (\Cref{fig:corr_renyi_alpha}, $\alpha^* \doteq 2.5$) and by \Cref{fig:freq_alphas_grid}, which shows that frequencies in the top percentile correlate \emph{negatively} with performance. 
Importantly, this finding does not contradict \citeposs{gowda2020finding} rule of thumb, which focuses on \emph{low} frequency tokens. 
While very high and very low frequencies produced by a tokenization scheme are not independent,
a tokenization scheme may feasibly produce both, neither or only one.

Furthermore, the Pearson correlation between the efficiency ($\ent_{2.5}/\ent_0$) and percentile frequency ($F_{0.03,0.83}$) is $0.96$, which suggests that both predictors are capturing the same underlying effect.\looseness=-1

\subsection{Experiment 2}
\label{sec:experiment_2}

\begin{figure}
\centering
\includegraphics[width=\linewidth]{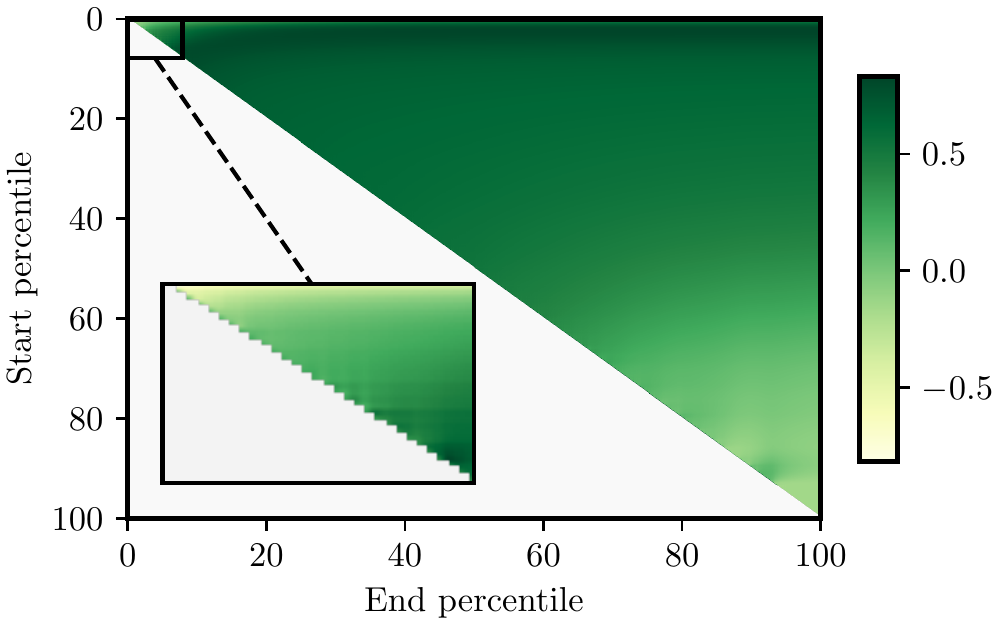}
\caption{Results for grid search over the best hyperparameters for percentile frequency predictor to maximize the absolute Pearson correlation.
The highest is for $3^\text{rd}$ to $83^\text{th}$ percentile with $\rho=0.81$.
}
\label{fig:freq_alphas_grid}
\vspace{-.6cm}
\end{figure}

In this experiment, we evaluate whether there exist aspects of a tokenization scheme that influence \bleu \emph{beyond} the Rényi efficiency. 
Following results in Experiment 1, we focus on Rényi efficiency at $\alpha=2.5$. 
In contrast to the first experiment, we consider different tokenization schemes (BPE, Unigram, WordPiece, LZW, Morfessor).
We manipulate their efficiency by lowering the amount of tokenizer training data (2k, 8k, 100k parallel lines) together with varying vocabulary sizes of 4k, 8k, and 16k tokens.
We then treat each tokenization-scheme--training-data-size--vocabulary-size triple as a single data point in this analysis.
We compare three different linear models \citep{gelman_hill_2006}, where \bleu is always the dependent variable: (i) with the tokenization scheme as a random effect, (ii) with  Rényi efficiency as a fixed effect, and (iii) with both. 
Importantly, we treat tokenization scheme as a random effect because the set of tokenization algorithms that we consider does not encompass all possible methods, i.e., only a sample of all possible algorithms are observed.\looseness=-1

To compare the ability of these different models to predict \bleu, we look at the average change in log-likelihood of held-out data points under a given model with respect to a baseline model: A model trained with only an intercept term. A larger value of $\Delta$ log-likelihood indicates that the data point is more probable under the comparison model, i.e., the comparison model more closely fits the observed data. We use 10-fold cross-validation to estimate these differences: Our data is split randomly into 10 folds, where 9 of the folds are used to learn model coefficients and the 10$^\text{th}$ fold is held back for evaluation. The same process is performed until we have a $\Delta$ log-likelihood value for each data point.\looseness=-1

\paragraph{Results.} In \cref{fig:dll_renyi_tokenizer}, we see that Rényi efficiency is a stronger predictor of MT performance than the tokenization scheme alone. Interestingly though, the predictive power of these two predictors seems to be orthogonal, as evinced by the mean $\Delta$ log-likelihood of a model with both predictors. This finding suggests that there are additional qualities of a good tokenization scheme that Rényi efficiency alone cannot capture. We leave the investigation of such qualities to future work.

\begin{figure}
\centering
\includegraphics[width=\linewidth]{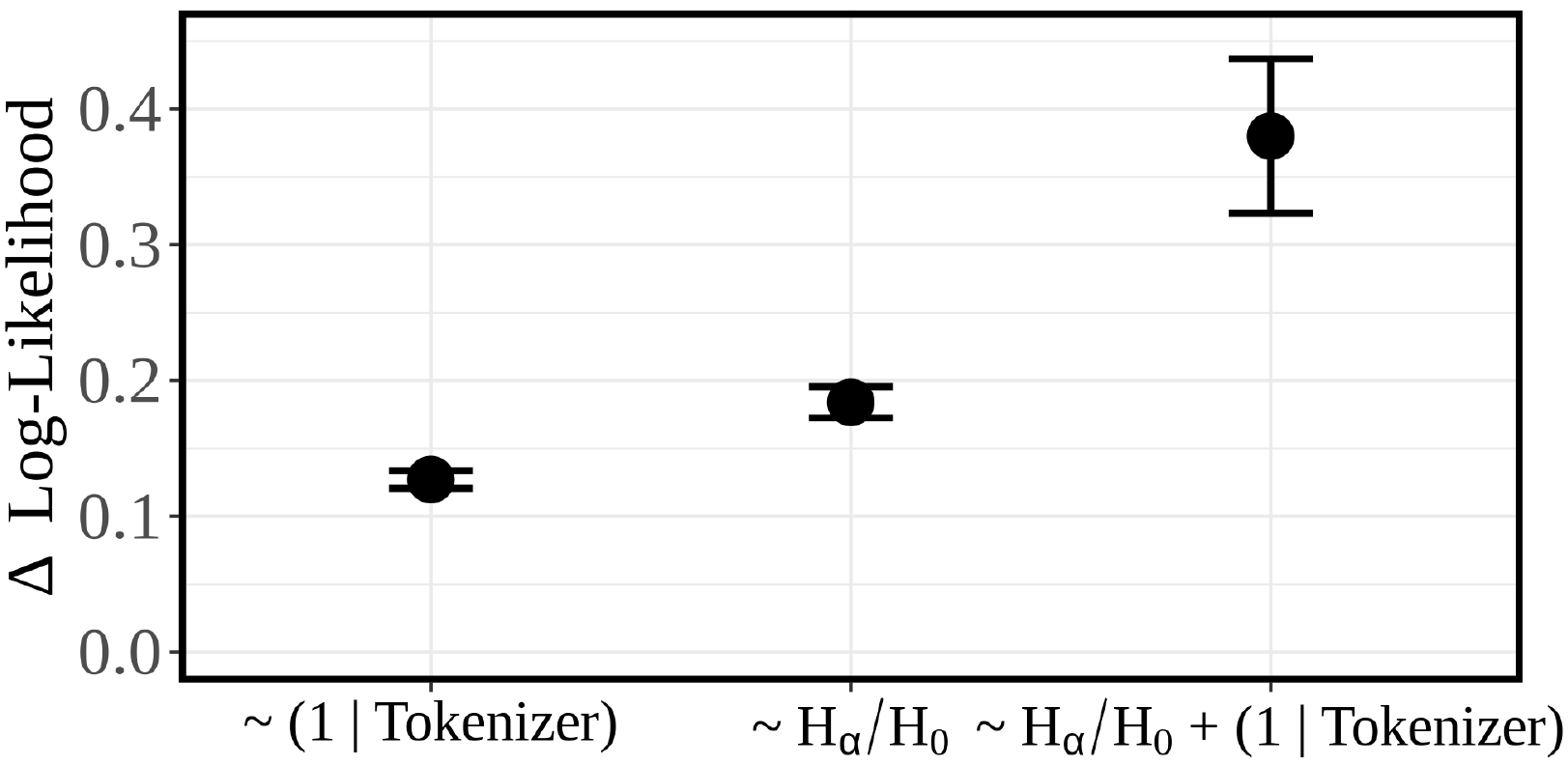}
\caption{Mean change in log-likelihood on held-out data under linear models using different predictors. Bars indicate 95\% confidence interval around the mean.
}
\label{fig:dll_renyi_tokenizer}
\vspace{-.6cm}
\end{figure}

\section{Conclusion}\label{sec:conclusion}
Our paper presents a new information-theoretic approach to characterizing a good tokenization scheme.
We contend that the Rényi efficiency of the unigram distribution that a tokenization scheme produces is a principled measure of the tokenization quality.
To test this claim, we evaluate a large set of tokenization schemes, with varying  vocabulary sizes and produced by different tokenization schemes. 
We observe how the Rényi efficiency of these different tokenizations relates to the performance of a downstream MT model. We find that, for an appropriate choice of the parameter $\alpha$, this new metric has a very strong Pearson correlation with \bleu: $0.78$ in comparison to just $-0.32$ for baseline sequence length. 
From a theoretical perspective, this property can be connected to a penalization of token distributions that are too unbalanced, having, in particular, very high-frequency tokens.
This finding is in line with the more general principle that compression is connected with learnability. 
Our framework also has practical benefits as it allows for an intrinsic evaluation of tokenization functions.\looseness=-1

\section*{Limitations}

It is possible that there is a hidden effect caused by the language pair direction, model selection, or training data and its size.
However, our results bear high statistical significance for cases where we desire high correlation and low statistical significance where we expect low correlation.
Assured by this and concerned by the large cost of training a large number of MT systems, we did not experiment with larger data or other language directions apart from limited additional experiments in \Cref{tab:comparison_otherlangs}.

\section*{Acknowledgements}
We would like to acknowledge Eleanor Chodroff and Shayne Sloggett for extensive discussion about the experimental design in Experiment 2.
Clara Meister was supported by the Google PhD Fellowship. 
Juan Luis Gastaldi has received funding from the European Union's Horizon 2020 research and innovation programme under grant agreement No 839730.

\bibliography{misc/bibliography}
\bibliographystyle{misc/acl_natbib}

\appendix

\onecolumn

\section{Related Work}
\label{sec:related}

Prior to the widespread adoption of subword tokenization, large vocabulary sizes (e.g., 500k) were needed to allow for output expressivity and to avoid a high proportion of out-of-vocabulary tokens. Various tricks were devised to tackle the resulting computational issues \citep{jean2015using,mi2016vocabulary,l2016vocabulary}.
On the other side of the spectrum, character-level NMT was also explored \citep{ling2015character,costa2016character}, though issues arise with large sequence lengths.
\citet{mielke2021between} provide an overview of the evolution of NLP tokenization and describe different types of tokenization approaches.
They conclude that reasoning about tokenizer choices remains a vital part of modern pipeline preparation.
In this context, our work quantifies and hence also automates some of this process by offering a framework to help guide the decision process and hyperparameter selection.
Somewhat similar to our work,  \citet{ataman2018evaluation} perform a comparison between BPE and Morfessor, though with only one specific vocabulary size (30k).
Similarly to \citet{gowda2020finding}, they suggest that homogeneity of token frequency is an important factor for MT model performance.

\section{Package Usage}
\label{sec:package_howto}

The package is open-source\footnote{\href{https://github.com/zouharvi/tokenization-scorer}{github.com/zouharvi/tokenization-scorer}} and can be downloaded via \texttt{pip} and used via the command-line:
\begin{leftbar1}
\begin{verbatim}
$ pip3 install tokenization-scorer
$ tokenization-scorer -i en-de.tokenized_with_unigramlm.{en,de}
> 0.4826

$ tokenization-scorer -i en-de.tokenized_with_wordpiece.{en,de}
> 0.5047
\end{verbatim}
\end{leftbar1}
\noindent or as a module in Python:
\begin{leftbar1}
\begin{verbatim}
import tokenization_scorer
text1 = "pick @@ed pick @@l @@ed pick @@les"
tokenization_scorer.score(text1, metric="renyi", power=2.5)
> 0.8031528501359657

text2 = "pick @@e @@d pick @@l @@e @@d pick @@l @@e @@s"
tokenization_scorer.score(text2, metric="renyi", power=2.5)
> 0.9105681923824472
\end{verbatim}
\end{leftbar1}
\noindent The supported metrics are the ones presented in \Cref{tab:comparison_corr_all}: \texttt{renyi\_efficiency} (default), \texttt{renyi\_entropy}, \texttt{shannon\_efficiency}, \texttt{shannon\_entropy}, \texttt{percentile\_freq}, \texttt{bits}, \texttt{sequence\_len}.
The power in Rényi can be modified using an extra parameter: \texttt{-e power=2.5}.
The similar applies to the percentile frequency with \texttt{-e perc\_start=0.03 perc\_end=0.83}.

\clearpage

\begin{table*}[htbp]
\centering
\resizebox{\linewidth}{!}{ 
\begin{tabular}{lcccc<{\hspace{4mm}}cccc}
\toprule
\textbf{Predictor}
& \multicolumn{4}{c}{\textbf{En$\rightarrow$De}}
& \multicolumn{4}{c}{\textbf{Cs$\rightarrow$En}}\\
\midrule
& \bleu & \chrf & \hspace{-2mm}\bleurt\hspace{-2mm} & \hspace{-2mm}\comet\hspace{-2mm}
& \bleu & \chrf & \hspace{-2mm}\bleurt\hspace{-2mm} & \hspace{-2mm}\comet\hspace{-2mm} \\
\midrule
Sequence length                & 0\% & 4\% & 23\% & 13\% & 4\% & 24\% & 51\% & 17\% \\
Percentile freq.               & 9\% & 21\% & 47\% & 33\% & 25\% & 63\% & 86\% & 55\% \\
Entropy                        & 0\% & 5\% & 24\% & 14\% & 7\% & 31\% & 58\% & 23\% \\
Entropy efficiency             & 22\% & 8\% & 0\% & 1\% & 6\% & 1\% & 3\% & 3\% \\
Rényi entropy                  & 7\% & 0\% & 6\% & 1\% & 0\% & 7\% & 24\% & 4\% \\
Rényi efficiency               & 53\% & 33\% & 12\% & 19\% & 32\% & 35\% & 17\% & 39\%  \\
\bottomrule
\end{tabular}
}
\caption{Variance explained between predictors and MT performance (\bleu, \chrf, \bleurt and \comet) in Experiment 1 (only 3 MT seeds, 5 temperatures and 4 vocabulary sizes) with different language directions.}
\label{tab:comparison_otherlangs}
\end{table*}

\section{Proofs}
\label{sec:proofs}

\begin{lemma}\label{lemma:cov}
Let $\pdeltastar$ be a distribution over $\alphabettwo^*$, and let $\pdelta$ be the unigram distribution induced by $\pdeltastar$ (\cref{def:unigram_distribution}). 
Then, the following equality holds\looseness=-1
\begin{subequations}
\begin{align}
\Lexpect \cdot \sum_{\delta \in \alphabettwo} \pdelta(\delta)   |\enc{\alphabettwo}(\delta)|+ \mathrm{Cov}\left( \avgencodinglen, \strlen \right) = \Lwencone
\end{align}
\end{subequations}
\end{lemma}
\begin{proof}
Let $\bdelta \in \alphabettwo^*$.
Define the \defn{expected counts} as follows
\begin{subequations}
\begin{align}
   \ecount &\defeq \sum_{\bdelta \in \alphabettwo^*} \pdeltastar(\bdelta)\, \countc(\delta, \bdelta) \\
   &=  \sum_{\bdelta \in \alphabettwo^*} \pdeltastar(\bdelta) \frac{\countc(\delta, \bdelta)}{|\bdelta|}|\bdelta| \\
   &= \expectdelta\left[ X_{\delta}(\bdelta) \cdot \strlen \right] 
\end{align}
\end{subequations}
We start by manipulating the expected code length %
\begin{subequations}
\begin{align}
\sum_{\bdelta \in \alphabettwo^*} \pdeltastar(\bdelta) |\enc{\alphabettwo}(\bdelta)| &= \sum_{\bdelta \in \alphabettwo^*} \pdeltastar(\bdelta) \sum_{n=1}^{|\bdelta|} |\enc{\alphabettwo}(\delta_n)|\label{eq:expect-cl} \\
&= \sum_{\bdelta \in \alphabettwo^*} \sum_{n=1}^{|\bdelta|} \pdeltastar(\bdelta) |\enc{\alphabettwo}(\delta_n)| \\
&= \sum_{\bdelta \in \alphabettwo^*} \sum_{\delta \in \alphabettwo} \pdeltastar(\bdelta) \countc(\delta, \bdelta) |\enc{\alphabettwo}(\delta)| \\
&= \sum_{\delta \in \alphabettwo}  |\enc{\alphabettwo}(\delta)| \left(\sum_{\bdelta \in \alphabettwo^*}\pdeltastar(\bdelta) \countc(\delta, \bdelta) \right)  \\
&= \sum_{\delta \in \alphabettwo} \ecount |\enc{\alphabettwo}(\delta)| 
\end{align}
\end{subequations}
Now, we proceed with algebraic manipulation.
\begin{subequations}
\begin{align}
\sum_{\delta \in \alphabettwo} \ecount &|\enc{\alphabettwo}(\delta)| = \sum_{\delta \in \alphabettwo} \expectdelta\left[ X_{\delta}(\bdelta) \cdot\strlen \right] |\enc{\alphabettwo}(\delta)| \\
&= \sum_{\delta \in \alphabettwo}\left( \expectdelta\left[ X_{\delta}(\bdelta) \right] \expectdelta\left[\strlen \right] |\enc{\alphabettwo}(\delta)| + \mathrm{Cov}
\left( X_{\delta}(\bdelta), \strlen \right)|\enc{\alphabettwo}(\delta)|\right) \\
\label{eq:other_direction}
&= \sum_{\delta \in \alphabettwo} \expectdelta\left[ X_{\delta}(\bdelta) \right] \expectdelta\left[\strlen\right] |\enc{\alphabettwo}(\delta)|+  \sum_{\delta \in \alphabettwo} \mathrm{Cov}
\left( X_{\delta}(\bdelta), \strlen\right)|\enc{\alphabettwo}(\delta)| \\
&= \sum_{\delta \in \alphabettwo} \expectdelta\left[ X_{\delta}(\bdelta) \right] \expectdelta\left[\strlen\right] |\enc{\alphabettwo}(\delta)|+ \mathrm{Cov}
\left( \sum_{\delta \in \alphabettwo} |\enc{\alphabettwo}(\delta)|  X_{\delta}(\bdelta), \strlen\right) \\
&= \sum_{\delta \in \alphabettwo} \expectdelta\left[ X_{\delta}(\bdelta) \right] \expectdelta\left[\strlen\right] |\enc{\alphabettwo}(\delta)|+ \mathrm{Cov}\Bigg(
 \underbrace{\sum_{\delta \in \alphabettwo} |\enc{\alphabettwo}(\delta)|  X_{\delta}(\bdelta)}_{\defeq \avgencodinglen}, \strlen\Bigg)\\
 &= \Lexpect \cdot \underbrace{\sum_{\delta \in \alphabettwo} \pdelta(\delta)   |\enc{\alphabettwo}(\delta)|}_{\text{expected unigram code length}}+ \mathrm{Cov}\left(
 \avgencodinglen, \strlen \right)
\end{align}
\end{subequations}
\end{proof}

\codebounded*

\begin{proof}
By \cref{lemma:cov}, we have
\begin{equation}
\Lexpect \cdot \sum_{\delta \in \alphabettwo} \pdelta(\delta)   |\encopt{\alphabettwo}(\delta)|+ \mathrm{Cov}\left( \avgencodinglenopt, \strlen \right) = \Lwenconeopt
\end{equation}
Now, by applying \cref{thm:shannon_source_coding}, we achieve
\begin{subequations}
\begin{align}
\Lexpect\cdot \ent(\rvDelta) & \leq \Lexpect \sum_{\delta \in \alphabettwo}  \pdelta(\delta) |\encopt{\alphabettwo}(\delta)| \\
 \Lexpect\cdot \ent(\rvDelta) + \mathrm{Cov}\left(
 \avgencodinglenopt, \strlen \right) &\leq \Lexpect \sum_{\delta \in \alphabettwo}  \pdelta(\delta)  
 |\encopt{\alphabettwo}(\delta)| + \mathrm{Cov}\left(
 \avgencodinglenopt, \strlen \right) \\
 \Lexpect\cdot \ent(\rvDelta) + \mathrm{Cov}\left(
 \avgencodinglenopt, \strlen \right) &\leq \Lwenconeopt 
\end{align}
\end{subequations}
Similarly, from \cref{thm:shannon_source_coding}, we have
\begin{subequations}
\begin{align}
\Lexpect\cdot \lceil \ent(\rvDelta)\rceil & \geq \Lexpect \sum_{\delta \in \alphabettwo}  \pdelta(\delta) |\encopt{\alphabettwo}(\delta)| \\
 \Lexpect\cdot \lceil \ent(\rvDelta)\rceil  + \mathrm{Cov}\left(
 \avgencodinglenopt, \strlen \right) &\geq \Lexpect \sum_{\delta \in \alphabettwo}  \pdelta(\delta)  
 |\encopt{\alphabettwo}(\delta)| + \mathrm{Cov}\left(
 \avgencodinglenopt, \strlen \right) \\
 \Lexpect\cdot \lceil \ent(\rvDelta)\rceil  + \mathrm{Cov}\left(
 \avgencodinglenopt, \strlen \right) &\geq \Lwenconeopt 
\end{align}
\end{subequations}
Putting these together with some additional algebraic manipulation we get
\begin{equation}
\ent(\rvDelta) \leq \frac{\Lwenconeopt - \mathrm{Cov}(\avgencodinglenopt,\strlen)}{\Lexpect}  \leq \lceil \ent(\rvDelta) \rceil
\end{equation}
This concludes the proof.
\end{proof}

\efficiencybound*
\begin{proof}
Recall that $\eff(\psigmastar,t) = \frac{\Lwenconeopt}{\mathcal{L}_{\encuni{\alphabettwo}}(\pdeltastar)}$.
To bound the denominator we use the fact, that for uniform distribution, $\mathrm{Cov}\left( \avgencodinglenuni, \strlen \right) = 0$ and obtain
$\Lexpect  \log |\alphabettwo| \leq \mathcal{L}_{\encuni{\alphabettwo}}(\pdeltastar) \leq \Lexpect  \lceil\log  |\alphabettwo| \rceil.$
Using this fact, we can obtain an upper bound
\begin{subequations}
\begin{align}
\eff(\psigmastar,t) &= \frac{\Lwenconeopt}{\mathcal{L}_{\encuni{\alphabettwo}}(\pdeltastar)} \\
&\leq \frac{\Lexpect \cdot \lceil \ent(\rvDelta)\rceil + \mathrm{Cov}\left( \avgencodinglenopt, \strlen \right)}{\mathcal{L}_{\encuni{\alphabettwo}}(\pdeltastar)} & \prooftext{(numerator lower-bound from \Cref{th:codebounded})}\\
&\leq \frac{\Lexpect \cdot \lceil \ent(\rvDelta) \rceil + \mathrm{Cov}\left( \avgencodinglenopt, \strlen \right)}{\Lexpect \log |\alphabettwo| } &\prooftext{(denominator upper-bound)}\\
 &= \frac{\lceil \ent(\rvDelta)\rceil +\frac{\mathrm{Cov}\left( \avgencodinglenopt, \strlen \right)}{\Lexpect}}{\log |\alphabettwo|}
\end{align}
\end{subequations}
and a corresponding lower bound
\begin{subequations}
\begin{align}
\eff(\psigmastar,t) &= \frac{\Lwenconeopt}{\mathcal{L}_{\encuni{\alphabettwo}}(\pdeltastar)} \\
&\geq \frac{\Lexpect \cdot  \ent(\rvDelta) + \mathrm{Cov}\left( \avgencodinglenopt, \strlen \right)}{\mathcal{L}_{\encuni{\alphabettwo}}(\pdeltastar)} & \prooftext{(numerator lower-bound from \Cref{th:codebounded})}\\
&\geq \frac{\Lexpect \cdot\ent(\rvDelta)  + \mathrm{Cov}\left( \avgencodinglenopt, \strlen \right)}{\Lexpect  \lceil \log |\alphabettwo|\rceil } &\prooftext{(denominator upper-bound)}\\
 &= \frac{ \ent(\rvDelta) +\frac{\mathrm{Cov}\left( \avgencodinglenopt, \strlen \right)}{\Lexpect}}{\lceil\log |\alphabettwo|\rceil}
\end{align}
\end{subequations}
\end{proof}

\begin{theorem}[Generalized Hölder's inequality; \S 2 in \citet{Aczel1980}]\label{thm:gen-holder}
Let $f$, $g$ and $h$ be vectors of positive values and coefficients $p$, $q$ and $r$ such that all but one are negative and $\frac{1}{p}+\frac{1}{q}+\frac{1}{r}=0$.
Further, if $\forall i: f_i g_i h_i = 1$ then
\begin{align}
\|f\|_p \, \|g\|_q \, \|h\|_r \leq 1. \label{eq:gen-holder}
\end{align}
\end{theorem}
As noted in \citet{Aczel1980}, \cref{thm:gen-holder} is in fact a simple special case of Theorem 12 in \citet{hardy1988inequalities}.
We will use \cref{thm:gen-holder} specifically with $r=-1$ and $h_i=f_i g_i$.
This simplifies the requirements for exactly one of $p$ and $q$ to be negative and $\frac{1}{p}+\frac{1}{q}=1$.
\Cref{eq:gen-holder} can then be restated as the following.
\begin{corollary}[Reverse Hölder's inequality]\label{thm:rev-holder}
If $f, g$ are positive vectors and $p,q$ are such that $\frac{1}{p}+\frac{1}{q}=1$ and exactly one is negative and the other is positive, then
\begin{align}
\|f\|_p \, \|g\|_q \leq \|f g\|_1. \label{eq:rev-holder}
\end{align}
\end{corollary}

\generalizedcampbell*
\begin{proof}
Let $p=-s$ and $q=1-\alpha$.  Let us consider three cases of $s$:
\begin{itemize}
\item $s = 0$, then $\alpha = 1$, then \cref{thm:shannon_source_coding} applies;
\item $s\in(-1,0)$, then $\alpha \in (1,\infty)$, $p \in (0,1)$ and $q \in (-\infty, 0)$.
\item $s\in(0,\infty)$, then $\alpha \in (0,1)$, $p \in (-\infty, 0)$ and $q \in (0, 1)$.
\end{itemize}
In the latter two cases, we have one of $p$ and $q$ is positive and one is negative. Further, our definitions of $p$ and $q$ imply that $p^{-1} + q^{-1} = 1$,  which allows us to use the reverse Hölder's inequality (\cref{thm:rev-holder}). 
First, we note that the  reverse Hölder's inequality implies that, for finite sequences $(x_i)$ and $(y_i)$,
\newcommand{\sumdelta}{\sum_{\delta\in\alphabettwo}}
\begin{align}
 \left(\sum x_i^p\right)^{\frac{1}{p}} \left(\sum y_i^q\right)^{\frac{1}{q}} &\leq \sum x_i y_i \\
 \left(\sum x_i^{-s}\right)^{-\frac{1}{s}} \left(\sum y_i^{1-\alpha}\right)^{\frac{1}{1-\alpha}} &\leq \sum x_i y_i & \prooftext{($p$ and $q$ substitution)}
\end{align}
Let $\ell(\delta) = |\enc{\alphabettwo}(\delta)|$ be the lengths of the codes given by our encoding. Further, let $\pdelta(\delta)$ be the unigram probabilities and $b$ be the base of our encoding. 
Now set $x_i = \pdelta(\delta)^{-\frac{1}{s}}b^{-\ell(\delta)}$ and $y_i=\pdelta(\delta)^{\frac{1}{s}}$.
This step is valid because both quantities are positive.
Then we proceed with algebraic manipulations
\begin{subequations}
\begin{align}
\left(\sumdelta \left(\pdelta(\delta)^{-\frac{1}{s}}b^{-\ell(\delta)}\right)^{-s}\right)^{-\frac{1}{s}} \left(\sumdelta (\pdelta(\delta)^{\frac{1}{s}})^{1-\alpha}\right)^{\frac{1}{1-\alpha}} &\leq \sumdelta (\pdelta(\delta)^{-\frac{1}{s}}b^{-\ell(\delta)}) (\pdelta(\delta)^{\frac{1}{s}}) \hspace{-2cm} \\
\left(\sumdelta (\pdelta(\delta)^{-\frac{1}{s}}b^{-\ell(\delta)})^{-s}\right)^{-\frac{1}{s}} \left(\sumdelta (\pdelta(\delta)^{\frac{1}{s}})^{1-\alpha}\right)^{\frac{1}{1-\alpha}} &\leq \sumdelta b^{-\ell(\delta)} & \prooftext{\,\,\,(algebra)} \\
\underbrace{\left(\sumdelta \pdelta(\delta) b^{s\cdot \ell(\delta)}\right)^{-\frac{1}{s}}}_{\defeq A} \underbrace{\left(\sumdelta \pdelta(\delta)^{\alpha}\right)^{\frac{1}{1-\alpha}}}_{\defeq B} &\leq \underbrace{\sumdelta b^{-\ell(\delta)}}_{\defeq C} & \prooftext{\,\,\,(algebra)}
\end{align}
\end{subequations}
Then, let $A \defeq \left(\sumdelta \pdelta(\delta) b^{s\cdot \ell(\delta)}\right)^{-\frac{1}{s}}$, $B \defeq \left(\sumdelta \pdelta(\delta)^{\alpha}\right)^{\frac{1}{1-\alpha}}$ and $C \defeq \sumdelta b^{-\ell(\delta)}$ in the subsequent proof descriptions.
Note that by the Kraft--McMillan inequality \citep{kraft1949device}, because our code is prefix-free (and $b > 0$), it must be that $0<C\leq1$.
\begin{subequations}
\begin{align}
 \frac{\left(\sumdelta \pdelta(\delta)^{\alpha}\right)^{\frac{1}{1-\alpha}}}{\sumdelta b^{-\ell(\delta)}} &\leq \left(\sumdelta \pdelta(\delta) b^{s\cdot \ell(\delta)}\right)^{\frac{1}{s}} &
\prooftext{(divide by $AC$)} \\
 \left(\sumdelta \pdelta(\delta)^{\alpha}\right)^{\frac{1}{1-\alpha}} &\leq \left(\sumdelta \pdelta(\delta) b^{s\cdot \ell(\delta)}\right)^{\frac{1}{s}} &
\prooftext{($0<C\leq1$)} \\
 \frac{1}{1-\alpha} \log \left(\sumdelta \pdelta(\delta)^{\alpha}\right) &\leq \frac{1}{s} \log \left(\sumdelta \pdelta(\delta) b^{s\cdot \ell(\delta)}\right) &
\prooftext{(take $\log$)} \\
 \ent_\alpha(\rvDelta)&\leq \mathcal{L}_\text{enc}^{(s)}(\pdelta)\label{eq:final}
\end{align}
\end{subequations}
Our constraint on $C$ is met and \cref{eq:final} holds with equality when the lengths of our code satisfy the following relationship
\begin{subequations}
\begin{align}
 b^{-\ell(\delta)} &= \frac{\pdelta(\delta)^\alpha}{\sum_{\delta' \in \alphabettwo} \pdelta(\delta')^\alpha}
& \prooftext{} \\
 \ell(\delta) &= -\alpha \log_b \pdelta(\delta)+ \log_b \left(\sum_{\delta' \in \alphabettwo} \pdelta(\delta')^\alpha\right)
& \prooftext{}
\label{thm_line:lemma_end}
\end{align}
\end{subequations}
Now we consider a sequence of $M$ tokens $\bdelta = \langle \delta_1, \ldots, \delta_M \rangle$, where each $\delta_m$ is sampled according to $\pdelta$; note that this is \emph{not} necessarily equivalent to a sequence $\bdelta\sim\pdeltastar$, as we do not assume independence between tokens in that setting. Let us first lift $\pdelta$ to take  sequences, i.e.,  $\pdelta(\bdelta) = \prod_{m=1}^M\pdelta(\delta)$, which follows naturally due to the independence of each $\delta_m$ in this setting. Now let 
\begin{align}\label{eq:Q-def}
    Q = \sum_{\bdelta\in \Delta^M} \pdelta(\bdelta)^\alpha= \left(\sumdelta \pdelta(\delta)^\alpha\right)^M
\end{align}
where the later equality follows from the fact that there are $|\alphabettwo|^M$ sequences in $\Delta^M$ and each $\delta$ appears an equal number of times. We will now use  \Cref{thm_line:lemma_end} to reason about the length of an optimal code for $\bdelta$, where we similarly denote length as $\ell(\bdelta)$. 
We will first assume that $s\in(-1,0)$ but later show that a slight modification to the proof makes it viable also for $s>0$. Following the result in \Cref{thm_line:lemma_end}, the length $\ell(\bdelta)$ of an optimal integer-length code for $\bdelta$ should satisfy
\begin{subequations}
\begin{align}
-\alpha \log_b \pdelta(\bdelta) + \log_b Q &\leq \ell(\bdelta) < -\alpha \log_b \pdelta(\bdelta) + \log_b Q+1
& \prooftext{} \\
-s \alpha \log_b \pdelta(\bdelta) + s \log_b Q &\geq s \cdot\ell(\bdelta) > s - s \alpha \log_b \pdelta(\bdelta) + s\log_b Q
& \prooftext{(multiply by $s\in(-1,0)$)} \label{thm_line:multiply_t_0} \\
 \pdelta(\bdelta)^{-s \alpha} Q^{s} &\geq b^{s \cdot\ell(\bdelta)} > b^{s} \pdelta(\bdelta)^{- s \alpha} Q^{s}
& \prooftext{(raise to power $b$)} \\
\pdelta(\bdelta)^{-s \alpha+1} Q^{s} &\geq \pdelta(\bdelta) b^{s \cdot\ell(\bdelta)} > b^{s} \pdelta(\bdelta)^{- s \alpha +1} Q^{s}
& \prooftext{(multiply by $\pdelta(\bdelta)$)} \\
\pdelta(\bdelta)^{\alpha} Q^{s} &\geq \pdelta(\bdelta) b^{s \cdot\ell(\bdelta)} > b^{s} \pdelta(\bdelta)^{\alpha} Q^{s}
& \prooftext{($s\alpha = 1-\alpha$)} \\
\sum_{\bdelta\in \Delta^M} \pdelta(\bdelta)^{\alpha}Q^{s} &\geq \sum_{\bdelta\in \Delta^M} \pdelta(\bdelta) b^{s \cdot\ell(\bdelta)} > Q^{s} b^{s} \sum_{\bdelta\in \Delta^M} \pdelta(\bdelta)^{\alpha} \nonumber
& \prooftext{} \\
& & \hspace{-3cm}\prooftext{(sum across all sequences of length $M$)}\\
Q^{s+1} &\geq \sum_{\bdelta\in \Delta^M} \pdelta(\bdelta) b^{s \cdot\ell(\bdelta)} > Q^{s+1} b^{s}
& \prooftext{(sub. $Q$)} \\
Q^{s+1} &\geq \left(\sumdelta \pdelta(\delta) b^{s \cdot\ell(\delta)} \right)^M > Q^{s+1} b^{s}
& \prooftext{(same logic as \cref{eq:Q-def})} \\
(s+1) \log_b Q &\geq M\log_b \left(\sumdelta \pdelta(\delta) b^{s \cdot\ell(\delta)}\right) > (s+1) \log_b Q + s\hspace{-2cm}
& \prooftext{(take $\log_b$)} \\
\frac{s+1}{s} \log_b Q &\leq \frac{M}{s} \log_b \left(\sumdelta \pdelta(\delta) b^{s \cdot\ell(\delta)}\right) < \frac{s+1}{s} \log_b Q + 1 \hspace{-3cm} & \prooftext{}\nonumber\\
&& \prooftext{(divide by $s\in (-1,0)$) }\label{thm_line:multiply_t_1}  \\
\frac{1}{1-\alpha} \log_b Q &\leq \frac{M}{s} \log_b \left(\sumdelta \pdelta(\delta)  b^{s \cdot\ell(\delta)}\right) < \frac{1}{1-\alpha} \log_b Q + 1\hspace{-2cm}
& \prooftext{(substitute $\alpha$)} \\
\frac{1}{1-\alpha} \log_b \left(\sumdelta \pdelta(\delta)^\alpha\right)^M \hspace{-4mm} &\leq \frac{M}{s} \log_b \left(\sumdelta \pdelta(\delta) b^{s \cdot\ell(\delta)}\right) < \frac{1}{1-\alpha} \log_b \left(\sumdelta \pdelta(\delta)^\alpha\right)^M + 1 \hspace{-5cm}& \prooftext{} \nonumber\\
& & \prooftext{(definition of $Q$)} \\
M \cdot \ent_\alpha(\rvDelta) &\leq M \cdot\Lwencdeltaopt{s} < M \cdot \ent_\alpha(\rvDelta) + 1
& \hspace{-2cm}\prooftext{}\\
\ent_\alpha(\rvDelta) &\leq \Lwencdeltaopt{s} < \ent_\alpha(\rvDelta) + \frac{1}{M}
& \prooftext{(divide by $M$)}
\end{align}
\end{subequations}
Note that in \Cref{thm_line:multiply_t_0}, we multiplied by a negative value and therefore swapped the directions of the inequalities. 
 Then, in \Cref{thm_line:multiply_t_1} we divided by a negative value and swapped the inequalities back to their original directions. 
If $s>0$, we would not change the directions of the inequalities and the proof would proceed as before.

For large $M$, we can get arbitrarily close to $\ent_\alpha$.
However, if we wish for $\ell(\delta)$ to be an integer, the upper bound becomes $\lceil \ent_\alpha(\rvDelta) \rceil$.\looseness=-1
\end{proof}

\penalizedcodbounded*
\begin{proof}
Let $\encopts{s}{\alphabettwo}$ be an $s$-optimal code, i.e., a code that minimizes $\Lwenc{s}$.
By \cref{lemma:cov}, we have
\begin{equation}
\Lexpect \cdot \sum_{\delta \in \alphabettwo} \pdelta(\delta)  |\encopts{s}{\alphabettwo}(\delta)|+ \mathrm{Cov}\left( \avgencodinglenopts{s}, \strlen \right) = \Lwenctwo{s}
\end{equation}
To prove the
bound, we proceed as follows.
Assume $s \neq 0$ as that case is covered by \cref{th:codebounded}.
We first start with a simple application of Jensen's inequality:
\begin{subequations}
\begin{align} 
\lceil \ent_\alpha(\rvDelta)\rceil  &\geq \Lwencdeltaopt{s} & \prooftext{(\cref{thm:generalized-campbell})}\\ 
&= \frac{\log_b\left(\sum_{\delta \in \alphabettwo} \pdelta(\delta) b^{s|\encopts{s}{\alphabettwo}(\delta)|}\right)}{s} & \prooftext{(definition)} \\ 
&\geq \frac{\sum_{\delta \in \alphabettwo} \pdelta(\delta) s|\encopts{s}{\alphabettwo}(\delta)| \log_b b}{s} & \prooftext{(Jensen's inequality)}\\
&= \sum_{\delta \in \alphabettwo} \pdelta(\delta) |\encopts{s}{\alphabettwo}(\delta)| 
\end{align}
\end{subequations}
Algebraically, combining the above results
\begin{equation}
\lceil \ent_\alpha(\rvDelta) \rceil \geq \frac{\Lwenctwo{s} - \mathrm{Cov}(\avgencodinglenopts{s},\strlen)}{\Lexpect} 
\end{equation}
which proves the theorem.
\end{proof}

\renyiefflower*
\begin{proof}
Recall from \Cref{def:eff_alpha} that $\eff_\alpha(\psigmastar, t) = \frac{\Lwenctwo{s}}{
\mathcal{L}_{\encuni{\alphabettwo}}(\pdeltastar)
}$, with $s=\frac{1}{a}-1$.
Again, to bound the denominator, we use the fact, that for uniform distribution, $\mathrm{Cov}\left(
\avgencodinglenuni
, \strlen \right) = 0$ and obtain $\Lexpect \log |\alphabettwo| \leq \mathcal{L}_{\encuni{\alphabettwo}}(\pdeltastar)$.
Then, we proceed as follows
\begin{subequations}
\begin{align}
\eff_\alpha(\psigmastar, t) &=\frac{\Lwenctwo{s}}{\mathcal{L}_{\encuni{\alphabettwo}}(\pdeltastar)} \\
&\leq \frac{\Lexpect \cdot \lceil \ent_\alpha(
\rvDelta
) \rceil +\mathrm{Cov}\left( \avgencodinglenopts{s}, \strlen \right)}{{\mathcal{L}_{\encuni{\alphabettwo}}(\pdeltastar)}} & \prooftext{(upper-bound numerator)} \\
&\leq \frac{\Lexpect \cdot \lceil \ent_\alpha(
\rvDelta
) \rceil +\mathrm{Cov}\left( \avgencodinglenopts{s}, \strlen \right)}{\Lexpect\cdot \log |\alphabettwo|} & \prooftext{(lower-bound denominator)} \\
&= \frac{\lceil \ent_\alpha(\rvDelta) \rceil +\frac{\mathrm{Cov}\left( \avgencodinglenopts{s}, \strlen \right)}{\Lexpect}}{\log |\alphabettwo| }
\end{align}
\end{subequations}
This proves the result.
\end{proof}

\begin{figure*}
\centering
\includegraphics[width=0.45\linewidth]{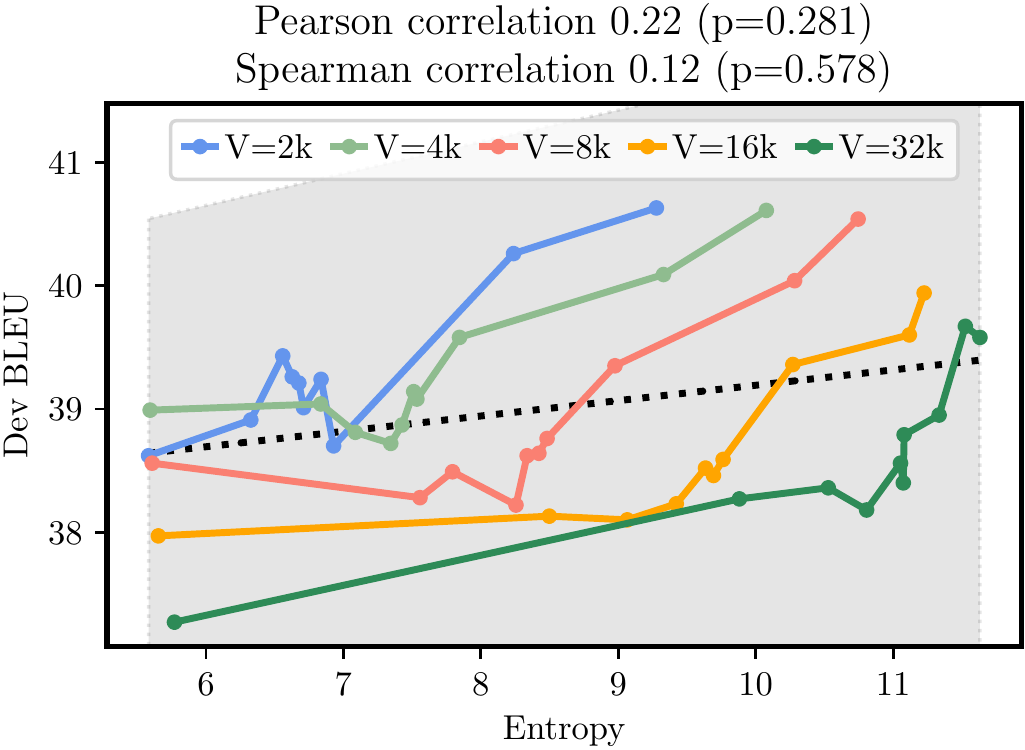}
\hspace{2mm}
\includegraphics[width=0.45\linewidth]{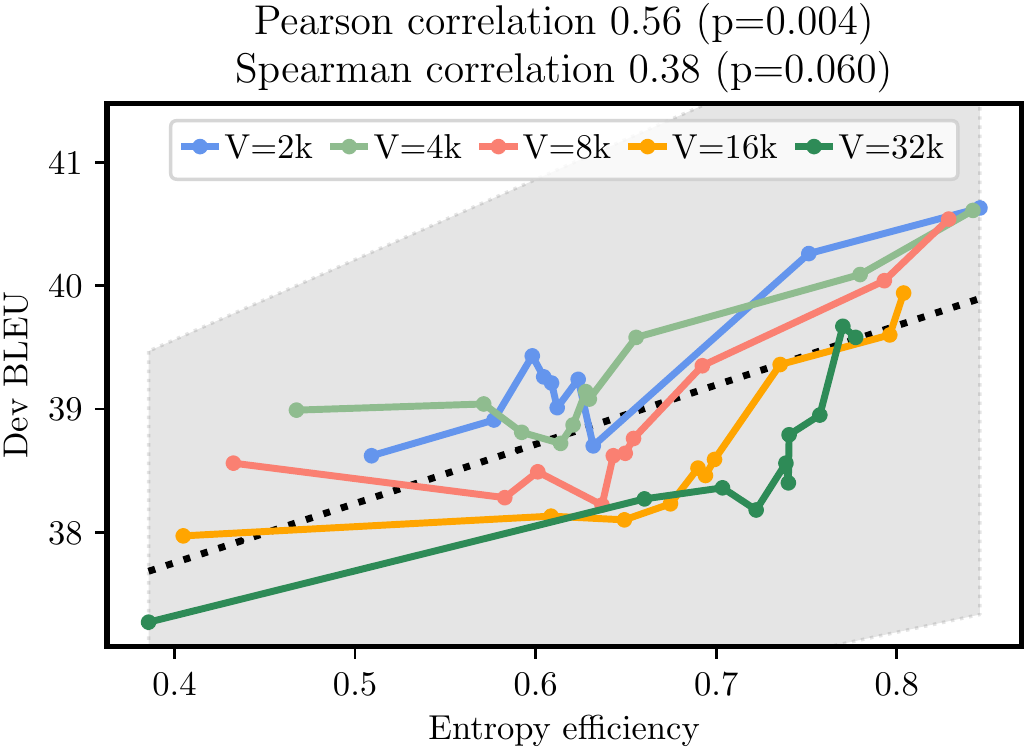}

\includegraphics[width=0.45\linewidth]{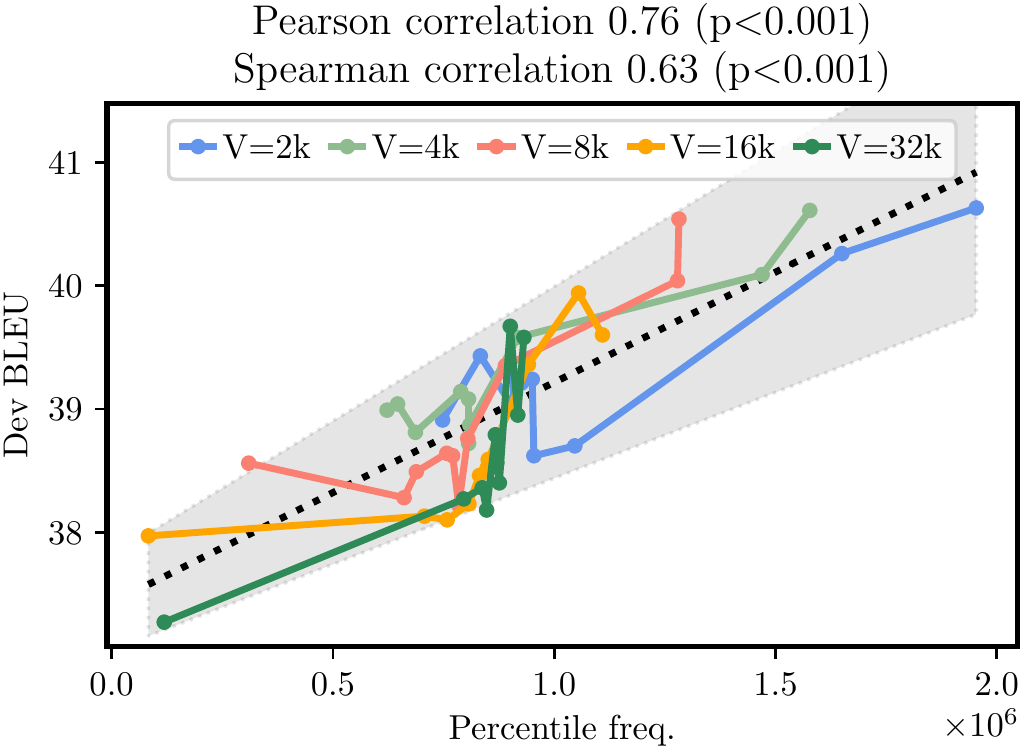}
\hspace{2mm}
\includegraphics[width=0.45\linewidth]{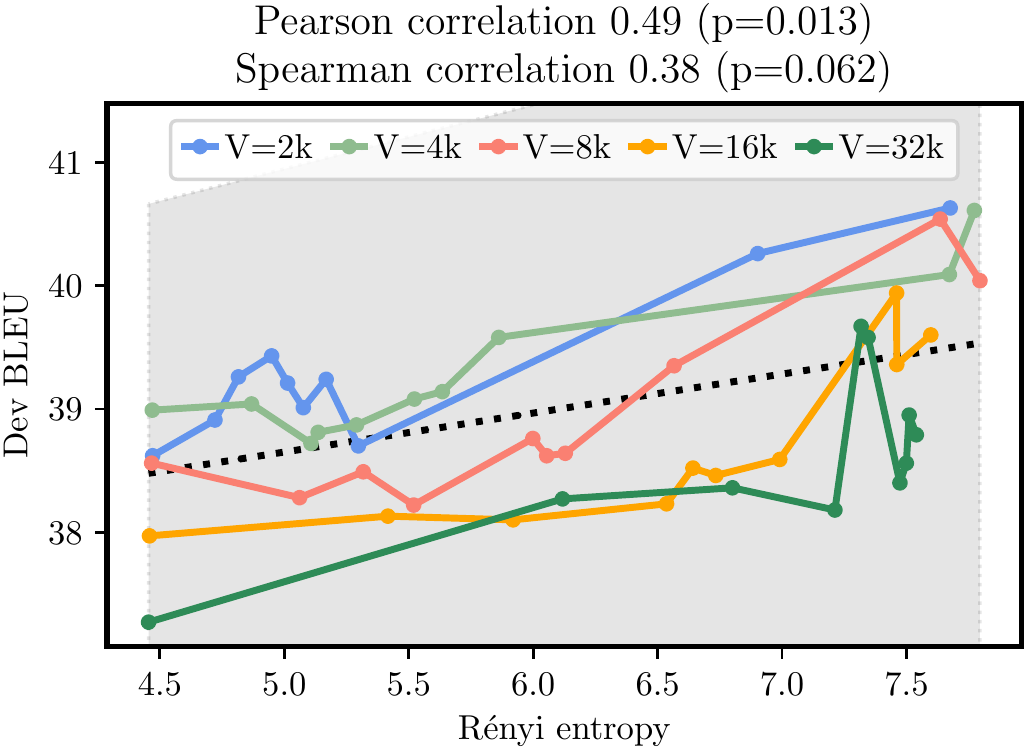}
\caption{Predictor visualization in parallel to \Cref{fig:corr_seq_len,fig:corr_renyi_eff}. Various features are used as predictors of MT performance (average of 5 runs). Bands show 95\% $t$-test confidence intervals for regression lines.}
\label{fig:other_predictor_visualization}
\end{figure*}

\section{Model, computation and reproducibility details}\label{sec:mt_details}

For the MT model, we use the \texttt{transformer\_iwslt\_de\_en} model architecture in Fairseq \citep{ott2019fairseq}.
For data, we use 1M training and 50k dev parallel sentences from English-German CommonCrawl \citep{elkishky_ccaligned_2020}.
BLEU evaluation performed using SacreBLEU \cite{post-2018-call}.
Details for the model training and the code to reproduce the experiments in this paper will be made publicly available.
For the first experiment we trained $5\times 6\times 9=270$ MT models.
For the second experiment we trained $5\times 5\times 3 \times 3=225$ MT models.
We used varying GPU models (GTX 1080, RTX 2080, RTX 3090) based on availability in shared compute cluster. Although different hardware and different tokenizations had an impact on the training time, the average per one configuration was 1.5 days.
Overall, we estimate 800 GPU days.

\section{Tokenization Schemes}
\label{sec:tokenizer_details}

In this section, we describe the four tokenization schemes used together in the second experiment.
Special attention is paid to BPE, which plays a role in the first experiment.
\Cref{tab:example_tokenization} shows how different tokenizers with varying vocabulary sizes tokenize the same word.
For simplicity, we do not include variable-length encoding, even though it has been previously applied to MT \citep{chitnis2015variable}.

\begin{table}
\centering
\begin{tabular}{l>{\ttfamily}l>{\ttfamily}l}
\toprule
\textbf{Tokenizer} & $\bm{V = 16k}$ & $\bm{V = 4k}$ \\
\midrule
BPE & Re gu lation & Re gu la tion \\
BPE $\tau = -0.4$ & Reg ul ation & Reg ul a tio n \\
Unigram & Regulation & Re gu l ation \\
WordPiece & Regul ation & Re gul ation \\
LZ & Re gul ation & Re gu la tion \\
Morfessor & Regul ation & Re gul ation \\
\bottomrule
\end{tabular}
\caption{Example tokenizations of the word \texttt{Regulation}.}
\label{tab:example_tokenization}
\end{table}

\subsection{Byte-Pair Encoding}
\label{subsec:bpe}

BPE was first discovered by \citet{gage1994} as a faster compression algorithm alternative to Lempel--Ziv--Welch.
It was later adapted by \citet{sennrich2016} as a tokenizer for MT.
The algorithm starts by splitting words of individual tokens of length 1 (characters).
Then it repeatedly takes the \textit{most frequent} pair of adjacent tokens and joins it into a new single token, adding this token to the vocabulary.
This procedure is repeated until we fill the predefined vocabulary budget.
\citet{formal_bpe} show that this greedy approach is approximately optimal when searching for the vocabulary (merge sequence).
We are however interested in intentionally suboptimal vocabularies.

\paragraph{Temperature.}

In order to introduce stochasticity into the process and intentionally alter the tokenization, instead of deterministically merging the most frequent token pair, we randomly sample a pair proportionally to their frequencies, similar to setup of \citet{saleva2023what}.
We add an additional hyperparameter to this strategy by annealing the frequency distribution with a temperature parameter $\tau$, where annealing is performed via a softmax. We use $ \tau = 0^+$ (original greedy BPE), $0.2$, $0.4$, $0.9$, $100$, $-100$, $-0.9$, $-0.4$, $-0.2$, and $0^-$.
Progressively, each temperature creates less and less optimal BPE compression model, increasing the encoded length of the data.
The last model, with $\tau=0^-$, we dub antigreedy because it always chooses the least frequent pair to merge, practically leading to a character-only model.

\subsection{Unigram LM}

While BPE repeatedly merges the most frequent pair, Unigram LM \citep{kudo2018subword} tokenizer modifies this part of the algorithm with a more complex approach.
The algorithm then jointly optimizes the token vocabulary and the unigram probability of the tokenized text.
Because of this dual objective, the algorithm is done in Expectation-Maximization manner.
We start by seeding the token vocabulary with the most frequent substrings.
In one step, probability is assigned to individual tokens and also how much worse the overall probability of the whole text would be if the single token was removed.
In the second step, top-$\eta$\% of the tokens is preserved.
These two steps are repeated until the vocabulary size is reduced to size $V$.\looseness=-1

\subsection{Linguistically Informed tokens}

Morfessor \citep{creutz2007unsupervised,virpioja2013morfessor,smit2014morfessor} is a family of unsupervised morphological analyzers that work on the minimum description principle, which is not distant from compression.
The link to tokenization is clear: Under some minimization objective, segment the input text into a list of morphemes with the constraint of at most $V$ distinct morphemes appearing in the whole text, where $V$ is a hyperparameter.
In the case of tokenization, we would use the morphemes as tokens.
For our purposes, we use vanilla Morfessor 2.0.6.

\subsection{Lempel--Ziv--Welch Compression}

BPE, a popular tokenization scheme used in many NLP applications, was first invented \citep{gage1994} as a faster compression algorithm alternative to Lempel--Ziv--Welch (LZW).
However, there is, in fact, a family of algorithms related to LZW.
LZ77 \citep{ziv1977universal} and LZW \citep{welch1984} are two of the most popular variants.

\begin{figure}
\centering
\includegraphics[width=0.33\linewidth]{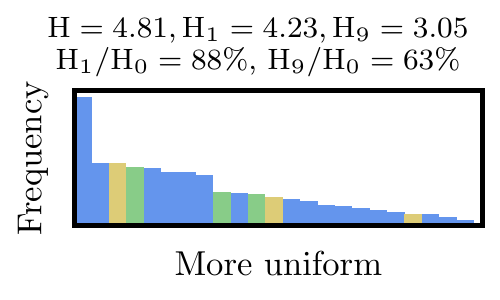}
\includegraphics[width=0.3\linewidth,trim=4mm 0 0 0,clip]{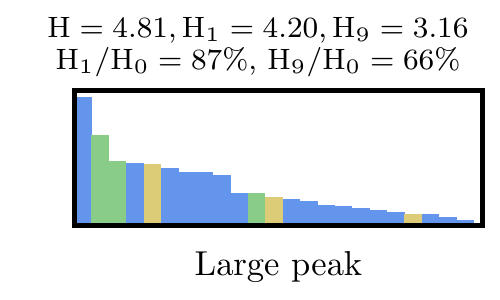}
\caption{Examples of unigram distributions of tokens with varying imbalance of probabilities. Letters corresponding to \textit{\texttt{the}} are in \colorsquare{\algcolorB} and those corresponding to \textit{\texttt{cow}} in \colorsquare{\algcolorC}.
}
\label{fig:peak_distributions}
\end{figure}

\vspace{1cm}

\begin{myexample}
\label{ex:peak_no_peak}
Consider a scenario where we want to encode a lowercased text without punctuation with the restriction of $|\alphabettwo|=28$.
Upon adding the 26 lowercased letters of the English alphabet and space, we can add one additional token.
Let us consider the tokens: \texttt{the} and \texttt{cow}.
Certainly, the former token is much more frequent in natural English distribution and hence a tokenizer fully utilizing this token would result in shorter optimal expected code length.
By adding the token \texttt{the}, we are lowering the probability mass of the most frequent English letters: \texttt{e} and \texttt{t}.
This does not happen with \texttt{cow} and hence the first distribution is more uniform.
For a depiction of this synthetic example, see \Cref{fig:peak_distributions}.
If we measure it using $\eff_1$, we obtain only $1\%$ difference but if we measure it using $\eff_9$, we get 3\% difference, showing that the first tokenization is much better.
\end{myexample}

\end{document}